\documentclass[runningheads]{llncs}

\usepackage[review,year=2024,ID=3877]{eccv}

\usepackage{eccvabbrv}

\usepackage{graphicx}
\usepackage{booktabs}
\usepackage[accsupp]{axessibility}  

\usepackage[pagebackref,breaklinks,colorlinks,citecolor=eccvblue]{hyperref}

\usepackage{orcidlink}
\usepackage{multirow}
\newcommand*{\ShowNotes}{}

\ifdefined\ShowNotes
  \newcommand{\colornote}[3]{{\color{#1}\bf{#2: #3}\normalfont}}
\else
  \newcommand{\colornote}[3]{}
\fi

\begin{document}
\def\thefootnote{*}\footnotetext{These authors contributed equally to this work.}\def\thefootnote{\arabic{footnote}}
\nolinenumbers
\title{The Hidden Attention of Mamba Models}

\titlerunning{The Hidden Attention of Mamba Models}

\author{Ameen Ali$^*$ \and Itamar Zimerman$^*$\and Lior Wolf}


\authorrunning{Ali, Zimerman, Wolf}

\institute{School of Computer Science, Tel Aviv University}

\maketitle

\begin{abstract}
The Mamba layer offers an efficient selective state space model (SSM) that is highly effective in modeling multiple domains, including NLP, long-range sequence processing, and computer vision. Selective SSMs are viewed as dual models, in which one trains in parallel on the entire sequence via an IO-aware parallel scan, and deploys in an autoregressive manner. We add a third view and show that such models can be viewed as attention-driven models. This new perspective enables us to empirically and theoretically compare the underlying mechanisms to that of the self-attention layers in transformers and allows us to peer inside the inner workings of the Mamba model with explainability methods. 
Our code is publicly available\footnote{\url{https://github.com/AmeenAli/HiddenMambaAttn}}.
\end{abstract}

\section{Introduction}
Recently, Selective State Space Layers~\cite{gu2023mamba}, also known as Mamba models, have shown remarkable performance in diverse applications including language modeling~\cite{gu2023mamba,mambamoe1,mambamoe2,wang2024mambabyte}, image processing~\cite{mambaViT1,mambaViT2}, video processing~\cite{mambaVideo}, medical imaging~\cite{mambaMedical7,mambaMedical5,mambaMedical8,mambaMedical4,mambaMedical3,mambaMedical2,mambaMedical1}, tabular data~\cite{ahamed2024mambatab}, point-cloud analysis~\cite{mambaPoint}, graphs~\cite{mambaGraph1,mambaGraph2}, N-dimensional sequence modeling~\cite{mambaND} and more. Characterized by their linear complexity in sequence length during training and fast RNN-like computation during inference (left and middle panels of  Fig.~\ref{fig:viewsOfmamba}), 
Mamba models offer a 5x increase in the throughput of Transformers for auto-regressive generation and the ability to efficiently handle long-range dependencies. 

Despite their growing success, the information-flow dynamics between tokens in Mamba models and the way they learn remain largely unexplored. Critical questions about their learning mechanisms, particularly how they capture dependencies and their resemblance to other established layers, such as RNNs, CNNs, or attention mechanisms, remain unanswered. Additionally, the lack of interoperability methods for these models may pose a significant hurdle to debugging them and may also reduce their applicability in socially sensitive domains in which explainability is required. 

Motivated by these gaps, our research aims to provide insights into the dynamics of Mamba models and develop methodologies for their interpretation. While the traditional views of state-space models are through the lens of convolutional or recurrent layers~\cite{gu2021combining}, we show that selective state-space layers are a form of attention models. This is achieved through a novel reformulation of Mamba computation using a data-control linear operator, unveiling hidden attention matrices within the Mamba layer. This enables us to employ well-established interpretability and explainability techniques, commonly used in transformer realms, to devise the first set of tools for interpreting Mamba models. Furthermore, our analysis of implicit attention matrices offers a direct framework for comparing the properties and inner representations of transformers~\cite{NIPS2017_3f5ee243} and Mamba models.

\textbf{Our main contributions} encompass the following main aspects: (i) We shed light on the fundamental nature of Mamba models, by showing that they rely on implicit attention, which is implemented by a unique data-control linear operator, as illustrated in Fig.~\ref{fig:viewsOfmamba} (right). (ii) Our analysis reveals that Mamba models give rise to three orders of magnitude 
more attention matrices than transformers. (iii) We provide a set of explainability and interpretability tools based on these hidden attention matrices. (iv) 
For comparable model sizes, Mamba model-based attention shows comparable explainability metrics results to that of transformers. 
(v) We present a theoretical analysis of the evolution of attention capabilities in state-space models and their expressiveness, offering a deeper understanding of the factors that contribute to Mamba's effectiveness.

 \begin{figure}[t]
\centering
\includegraphics[width=0.92\textwidth]{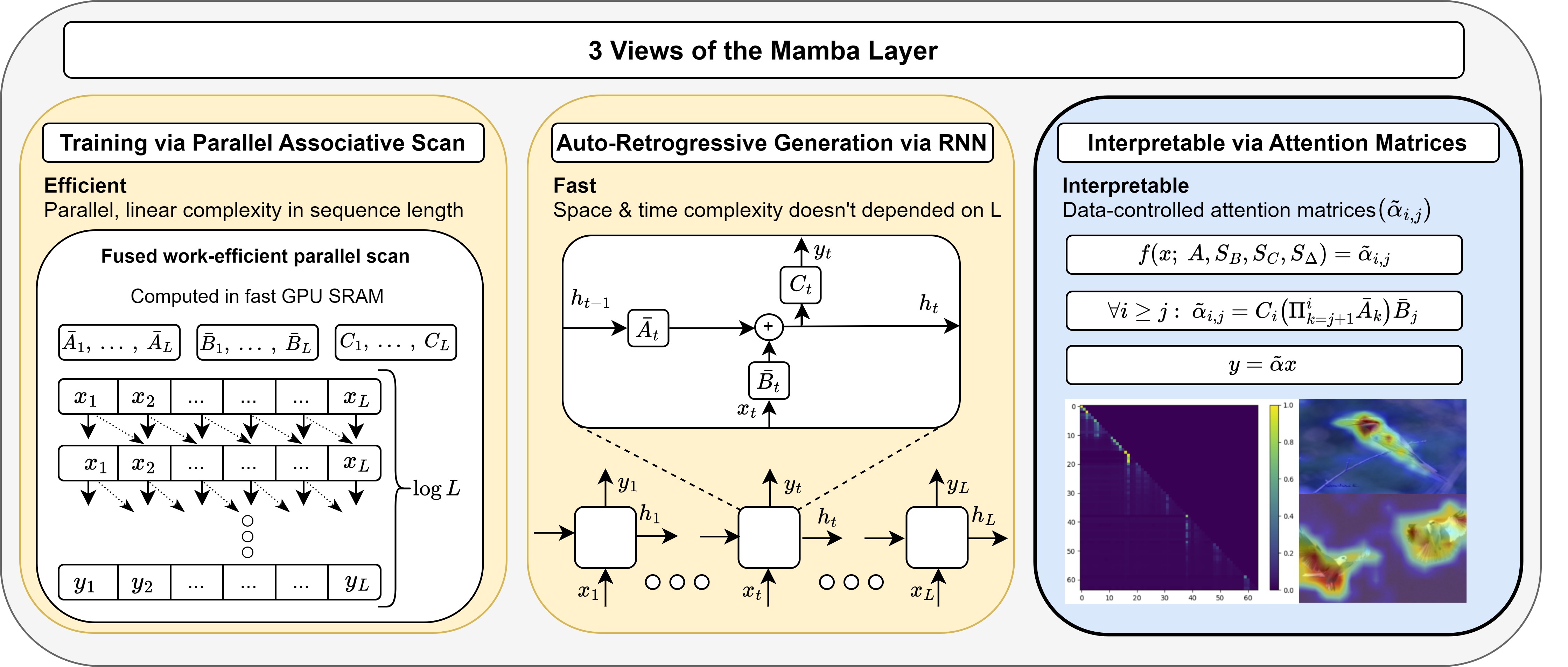}
\caption{Three Perspectives of the Selective State-Space Layer:\textbf{(Left)} Selective State-Space Models (SSMs) can be efficiently computed with linear complexity using parallel scans, allowing for effective parallelization on modern hardware, such as GPUs. \textbf{(Middle)} Similar to SSMs, the selective state-space layer can be computed via a time-variant recurrent rule. \textbf{(Right)} A new view of the selective SSM layer, showing that it uses attention similarly to transformers (see Eq.~\ref{eq:MAMbaASmatmul}). Our view enables the generation of attention maps, offering valuable applications in areas such as XAI.}
\label{fig:viewsOfmamba}
\end{figure}

\section{Background}

{\noindent\textbf{Transformers\quad}}%
The Transformer architecture~\cite{NIPS2017_3f5ee243} is the dominant architecture in the recent NLP and Computer Vision literature. It relies on self-attention to capture dependencies between different tokens. Self-attention allows these models to dynamically focus on different parts of the input sequence, calculating the relevance of each part to others. It can be computed as follows:
\begin{equation} \label{eq:attnMAT}
  Self-Attention(Q, K, V) = \alpha V, \quad \alpha = \text{softmax}\left(\frac{QK^T}{\sqrt{d_k}}\right)
\end{equation}
where \(Q\), \(K\), and \(V\) represent queries, keys, and values, respectively, and \(d_k\) is the dimension of the keys. Additionally, the Transformer utilizes $H$ attention heads to process information in parallel, allowing the model to capture various dependencies. The attention matrix $\alpha$ enables the models to weigh the importance of tokens based on their contribution to the context, and they can also used for interpretability~\cite{bahdanau2014neural}, explainability~\cite{chefer2021transformer}, and improved classification~\cite{touvron2021training,chefer2022optimizing}. 

\medskip
{\noindent\textbf{State-Space Layers\quad}}%
State-Space Layers were first introduced in~\cite{gu2021combining} and have seen significant improvements through the seminal work in ~\cite{gu2021efficiently}. These layers have demonstrated promising results across several domains, including NLP~\cite{mehta2022long,fu2022hungry}, audio generation~\cite{goel2022s}, image processing~\cite{yan2023diffusion, baron20232,nguyen2022s4nd}, long video understanding~\cite{wang2023selective}, RL~\cite{david2022decision,lu2024structured}, speech recognition~\cite{saon2023diagonal}, and more. 
Given one channel of the input sequence $x := (x_1, \cdots, x_L)$ such that $x_i \in \mathbb{R}$, these layers can be implemented using either recurrence or convolution. The recurrent formulation, which relies on the recurrent state $h_t \in \mathbb{R}^N$ where $N$ is the state size, is defined as follows: given the discretization functions $f_A, f_B$, and parameters $A$, $B$, $C$ and $\Delta$, the recurrent rule for the SSM is:
\begin{equation} \label{eq:ssmreccurent}
\bar{A} = f_A(A,\Delta),\quad \bar{B} = f_B(A, B,\Delta), \quad h_t = \bar{A} h_{t-1} + \bar{B} x_t, \quad y_t = C h_t
\end{equation}
This recurrent rule can be expanded as:
\begin{equation} \label{eq:closedFormSSM}
    h_t = \bar{A}^t \bar{B} x_0 +  \bar{A}^{t-1} \bar{B} x_1+ \cdots + \bar{B} x_t, \quad
y_t =  C \bar{A}^t \bar{B} x_0 +  C \bar{A}^{t-1} \bar{B} x_1+ \cdots + C \bar{B} x_t
\end{equation}

Since the recurrence is linear, Eq.~\ref{eq:closedFormSSM} can also be expressed as a convolution, via a convolution kernel $K := (k_1, \cdots, k_L)$, where $k_i = C \bar{A}^{i-1} \bar{B}$, thus allowing sub-quadratic complexity in sequence length.
The equivalence between the recurrence and the convolution provides a versatile framework that enables parallel and efficient training with sub-quadratic complexity with the convolution view, alongside a faster recurrent view, facilitating the acceleration of autoregressive generation by decoupling step complexity from sequence length. As the layer defined as a map from $\mathbb{R^L}$ to $\mathbb{R^L}$, to process $D$ channels the layer employs $D$ independent copies of itself. 

\medskip
{\noindent\textbf{S6 Layers\quad}}%
A recent development in state space layers is selective SSMs~\cite{gu2023mamba} (S6), which show outstanding performance in NLP~\cite{mambamoe1,mambamoe2,wang2024mambabyte}, vision~\cite{mambaViT1,mambaViT2}, graph classification~\cite{mambaGraph1,mambaGraph2}, and more. These models rely on time-variant SSMs, namely, the discrete matrices $\bar{A},\bar{B},$ and $C$ of each channel are modified over the $L$ time steps depending on the input sequence. As opposed to traditional state-space layers, which operate individually on each channel, selective state-space layers compute the SSM matrices $\bar{A}_i, \bar{B}_i, C_i$ for all $i \leq L$ based on all the channels, and then apply the time-variant recurrent rule individually for each channel. Hence, we denote the entire input sequence by $\hat{x} := (\hat{x}_1, \cdots, \hat{x}_L) \in \mathbb{R}^{L \times D}$ where $\hat{x}_i \in \mathbb{R}^{D}$. The per-time discrete matrices $\bar{A_i}, \bar{B_i},$ and $C_i$ are defined as follows:

\begin{equation} \label{eq:TimeVariantMatrices1}
    B_i = S_B (\hat{x}_i), \quad C_i = S_C (\hat{x}_i), \quad \Delta_i = \text{softplus}(S_{\Delta}(\hat{x}_i))
\end{equation}

\begin{equation} \label{eq:discretization}
     f_A(\Delta_i, A) = \exp (\Delta_i A), \quad f_B(\Delta_i, A, B_i) = \Delta_i B_i, 
\end{equation}
\begin{equation}\label{eq:TimeVariantMatrices2}
    \bar{A}_i = f_A(\Delta_i, A), \quad \bar{B}_i = f_B(\Delta_i, A, B_i)
\end{equation}
where $f_A, f_B$ represents the discretization rule, $S_B, S_C, S_{\Delta}$ are linear projection layers, and SoftPlus is an elementwise function that is a smooth approximation of ReLU. While previous state-space layers employ complex-valued SSMs and non-diagonal matrices, Mamba employs real-diagonal parametrization. 

The motivation for input-dependent time-variant layers is to make those recurrent layers more expressive and flexible, allowing them to capture more complex dependencies. While other input-dependent time-variant mechanisms have been proposed in previous works through gated RNNs, the S5 layer~\cite{smith2022simplified}, or adaptive filtering via input-dependent IIR filters~\cite{lutati2023focus}, Mamba significantly improves on these layers by presenting a flexible, yet still efficient, approach. This efficiency was achieved via the IO-aware implementation of associative scans, which can be parallelized on modern hardware via work-efficient parallel scanners~\cite{blelloch1990prefix,scan2}.

\medskip
{\noindent\textbf{Mamba\quad}}%
The Mamba block is built on top of the selective state-space layer, Conv1D and other elementwise operators. Inspired by the architecure of Gated MLP and H3~\cite{fu2022hungry}, and given an input $\hat{x}' := (\hat{x}'_1, \cdots \hat{x}'_L)$ it is defined as follows:
\begin{equation}
    \hat{x} = \text{SiLU( Conv1D( Linear(}\hat{x}'\text{) ) )}, \quad \hat{z} = \text{SiLU( Linear(}\hat{x}'\text{) )} 
\end{equation}
\begin{equation}
    \hat{y}' = \text{Linear}(\text{Selective SSM}(\hat{x}) \otimes \hat{z})), \quad 
    \hat{y}= \text{LayerNorm}(\hat{y}' + \hat{x}')
\label{eq:output_block}
\end{equation}
where $\otimes$ is elementwise multiplication. Mamba models contain $\Lambda$ stacked mamba blocks and $D$ channels per block, and we denote the tensors in the i-th block and j-th channel with a superscript, where the first index refers to the block number.

Inspired by the vision transformer ViT~\cite{dosovitskiy2020image}, both~\cite{mambaViT1,mambaViT2} replace the standard self-attention mechanism by two Mamba layers, where each layer is applied in a bidirectional manner. The resulting model achieves favorable results compared to the standard ViT in terms of both accuracy and efficiency, when comparing models with the same number of parameters.

\medskip
{\noindent\textbf{Explainability\quad}}
Explainability methods have been extensively explored in the context of deep neural networks, particularly in domains such as natural language processing (NLP)\cite{abnar2020quantifying,arras-etal-2017-explaining,yuan2021explaining,ali2022xai,chefer2021transformer,chefer2021generic}, computer vision\cite{selvaraju2017grad,nam2020relative,bach2015pixel,gur2021visualization,sundararajan2017axiomatic}, and attention-based models~\cite{ali2022xai,chefer2021transformer,chefer2021generic,yuan2021explaining}. 

The contributions most closely aligned with ours are those specifically tailored for transformer explainability. Abnar and Zuidema~\cite{abnar2020quantifying} {introduce the {Attention-Rollout} method, which aggregates the attention matrices across different layers by analyzing the paths in the inter-layer pairwise attention graph.} 
Chefer et al.~\cite{chefer2021transformer,chefer2021generic} combine LRP scores~\cite{bach2015pixel} with the attention gradients to obtain {class-specific} relevance scores. Ali et al.~\cite{ali2022xai} enhanced attributions by treating the non-linear Softmax and LayerNorm operators as a constant, thereby attributing relevance exclusively through the value path, disregarding these operators. Yuan et al.~\cite{yuan2021explaining} treats the output token representations as states in a Markov chain in which the transition matrix is built using attention weights.

Our work performs similar attention-based analysis for Mamba and we derive versions of~\cite{abnar2020quantifying,chefer2021transformer} that are suitable for such SSM models.

\section{Method}
In this section, we detail our methodology. First, in section~\ref{sec:S6asAttn}, we reformulate selective state-space (S6) layers as self-attention, enabling the extraction of attention matrices from S6 layers. Subsequently, in sections~\ref{sec:AttnMatForXAI} and~\ref{sec:AttrForXAI}, we demonstrate how these hidden attention matrices can be leveraged to develop class-agnostic and class-specific tools for explainable AI of Mamba models.

\subsection{Hidden Attention Matrices In Selective State Spaces Layers\label{sec:S6asAttn}}
Given the per-channel time-variant system matrices $\bar{A}_1, \cdots, \bar{A}_L$, $\bar{B}_1, \cdots, \bar{B}_L$, and $C_1, \cdots, C_L$ from Eq.~\ref{eq:TimeVariantMatrices1} and~\ref{eq:TimeVariantMatrices2}, each channel within the selective state-space layers can be processed independently. Thus, for simplicity, the formulation presented in this section will proceed under the assumption that the input sequence $x$ consists of a single channel.
 
By considering the initial conditions $h_{0}=0$, unrolling Eq.~\ref{eq:ssmreccurent} 
yields: 

\begin{equation} \label{eq:unrolling1}
 h_1 =  \bar{B}_1  x_1 ,\quad y_1 = C_1 \bar{B}_1 x_{1},\quad  h_2 = \bar{A}_2 \bar{B}_1 x_{1} + \bar{B}_2 x_{2}  ,\quad y_2 = C_2 \bar{A}_2 \bar{B}_1 x_{1} + C_2 \bar{B}_2 x_{2} 
\end{equation}


and in general:

\begin{equation} \label{eq:unrolling3}
 h_t = \sum_{j=1}^t \big{(} \Pi_{k=j+1}^t \bar{A}_k \big{)} \bar{B}_{j}x_j, %
 \quad %
 y_t = C_t\sum_{j=1}^t \big{(} \Pi_{k=j+1}^t \bar{A}_k \big{)} \bar{B}_{j}x_j
\end{equation}

By converting Eq.~\ref{eq:unrolling3} into a matrix form we get:

\begin{equation}\label{eq:MAMbaASmatmul}
y = \tilde{\alpha} x, \quad
\begin{bmatrix}
y_1 \\
y_2\\ 
\vdots \\
y_L \\
\end{bmatrix} 
=
\begin{bmatrix}
    C_1 \bar{B}_1 & 0 & \cdots & 0 \\
    C_2 \bar{A}_2 \bar{B}_1 & C_2 \bar{B}_2 & \cdots & 0 \\
    \vdots & \vdots & \ddots & 0 \\
    C_L \Pi_{k=2}^L \bar{A}_k \bar{B}_{1} \quad & C_L \Pi_{k=3}^L \bar{A}_k \bar{B}_{2} \quad & \cdots \quad & C_L \bar{B}_L
\end{bmatrix}
\begin{bmatrix}
x_1 \\
x_2\\ 
\vdots \\
x_L \\
\end{bmatrix}
\end{equation}

Hence, the S6 layer can be viewed as a data-controlled linear operator~\cite{poli2023hyena}, 
where the matrix $\tilde{\alpha} \in \mathbb{R}^{L \times L}$ is a function of the input and the parameters $A, S_B, S_C ,S_\Delta$. The element at row $i$ and column $j$ captures how $x_j$ influences $y_i$, and is computed by:
\begin{equation}\label{eq:attnPerlocation}
    \tilde{\alpha}_{i,j} = C_i \Big{(}\Pi_{k=j+1}^i \bar{A}_k \Big{)} \bar{B}_j
\end{equation}

Eq.~\ref{eq:MAMbaASmatmul} and \ref{eq:attnPerlocation} link $\tilde{\alpha}$ to the conventional standard attention matrix (Eq.~\ref{eq:attnMAT}), and highlight that S6 can be considered a variant of causal self-attention.

\medskip
\noindent{\textbf{Simplifying and Interpreting the Hidden Matrices }}
Since $\bar{A_t}$ is a diagonal matrix, the different $N$ coordinates of the state $h_t$ in Eq.~\ref{eq:unrolling3} do not interact when computing $h_{t+1}$. Thus, Eq.~\ref{eq:unrolling3} (left) can be computed independently for each coordinate $m \in \{1,2,\dots,N\}$:
\begin{equation} \label{eq:unrolling3cord}
 h_t[m] =\sum_{j=1}^t \big{(} \Pi_{k=j+1}^t \bar{A}_k[m,m] \big{)} \bar{B}_{j}[m] x_j,\quad
 y_t = \sum_{m=1}^N C_t[m]h_t[m]
\end{equation} 
where $C_i[m],A_k[m,m],B_j[m] \in \mathbb{R}$, plugging it into Eq.~\ref{eq:attnPerlocation} yields:

\begin{equation}\label{eq:attnPerCord}
    \tilde{\alpha}_{i,j} = C_i \Big{(} \Pi_{k=j+1}^i  \bar{A}_k \Big{)} \bar{B}_j= \quad \sum_{m=1}^N C_i[m]  \Big{(}\Pi_{k=j+1}^i \bar{A}_k[m,m] \Big{)} \bar{B}_j [m]
\end{equation}

Note that while equations~\ref{eq:unrolling3} and~\ref{eq:attnPerlocation} contain matrix multiplication, Eq.~\ref{eq:unrolling3cord} relies on elementwise multiplication.

An interesting observation arising from Eq.~\ref{eq:attnPerCord} is that a single channel of S6 produces $N$ inner attention matrices $C_i[m]  \Big{(}\Pi_{k=j+1}^i \bar{A}_k[m,m]  \Big{)}\bar{B}_j[m]$, which are summed up over $m$ to obtain $\tilde{\alpha}$.  
In contrast, in the Transformer, a single attention matrix is produced by each of the $H$ attention heads. Given that the number of channels in Mamba models $D$ is typically a hundred times greater than the number of heads in a transformer (for example, Vision-Mamba-Tiny has $D=384$ channels, compared to $H=3$ heads in DeiT-Tiny), the Mamba layer generates approximately $\frac{DN}{H}\approx 100N$ more attention matrices than the original self-attention layer.

To further understand the structure and characterization of these attention matrices, we will express the hidden attention matrices $\tilde{\alpha}$ for each channel $d$ as a direct function of the input $\hat{x}$. To do so, we first substitute Eq.\ref{eq:TimeVariantMatrices1},~\ref{eq:discretization} and Eq.\ref{eq:TimeVariantMatrices2} into Eq.~\ref{eq:attnPerlocation}, and obtain:

\begin{equation}\label{eq:attnFinal}
    \tilde{\alpha}_{i,j} =
    S_{C} (\hat{x}_i) \Big{(} \Pi_{k=j+1}^i \exp\Big{(}\text{softplus}(S_{\Delta}(\hat{x}_k)) A \Big{)}\Big{)}  \text{softplus}(S_{\Delta}(\hat{x}_j)) S_{B}  (\hat{x}_j) =
\end{equation}
\begin{equation}\label{eq:attnFinal2}
    S_{C} (\hat{x}_i)  \Big{(} \exp\ \Big{(} \sum_{k=j+1}^i \text{softplus}(S_{\Delta}(\hat{x}_k)) \Big{)} A \Big{)}  \text{softplus}(S_{\Delta}(\hat{x}_j)) S_{B}  (\hat{x}_j)
\end{equation}

For simplicitly, we propose a simplification of Eq.~\ref{eq:attnFinal2} by substituting the softplus function with the ReLU function, and summing only over positive elements:

\begin{equation}\label{eq:attnFinal2Simplified}
    \tilde{\alpha}_{i,j} \approx S_{C} (\hat{x}_i)  \left( \exp\ \left( \sum_{\substack{k=j+1 \\ S_{\Delta}(\hat{x}_k) > 0}}^i S_{\Delta}(\hat{x}_k) \right) A \right)  \text{ReLU}(S_{\Delta}(\hat{x}_j)) S_{B}  (\hat{x}_j)
\end{equation}

Consider the following query/key/value notation:
\begin{equation} \label{eq:snipleAttnNotations}
    \tilde{Q_i} := S_{C} (\hat{x}_i),  \textbf{ }
    \tilde{K_j} :=  \text{ReLU}(S_{\Delta}(\hat{x}_j)S_{B}  (\hat{x}_j),  \textbf{ } 
    \tilde{H}_{i,j} := \exp \Big{(} \sum_{\substack{k=j+1 \\ S_{\Delta}(\hat{x}_k) > 0}}^i S_{\Delta}(\hat{x}_k) \Big{)} A
\end{equation}
Eq.~\ref{eq:attnFinal2Simplified} can be further simplified to:
\begin{equation} \label{eq:snipleAttn}
    \tilde{\alpha}_{i,j} \approx \tilde{Q_i} \tilde{H}_{i,j}  \tilde{K_j} 
\end{equation}

This formulation enhances our understanding of the Mamba's attention mechanism. Whereas traditional self-attention captures the influence of \(x_j\) on \(x_i\) through the dot products between \(Q_i\) and \(K_j\), Mamba's approach correlates this influence with \(\tilde{Q}_i\) and \(\tilde{K}_j\), respectively. Additionally, \(\tilde{H}_{i,j}\) controls the significance of the recent \(i-j\) tokens, encapsulating the continuous aggregated historical context spanning from \(x_j\) to \(x_i\). 

This distinction between self-attention and Mamba, captured by \(\tilde{H}_{i,j}\) could be a key factor in enabling Mamba models to understand and utilize \underline{continuous} historical context within sequences more efficiently than attention.

Furthermore, Eq.~\ref{eq:snipleAttn}, and~\ref{eq:snipleAttnNotations} offer further insights into the characterization of the hidden attention matrices by demonstrating that the only terms modified across channels are $A$ and $\Delta_i$, which influence the values of \(\tilde{H}_{i,j}\) and \(\tilde{K}_j\) through the discretization rule in Eq.~\ref{eq:discretization}. Hence, all the hidden attention matrices follow a \underline{common pattern}, distinguished by the keys \(\tilde{K}_j\) via $\Delta_i$ and the significance of the history \(\tilde{H}_{i,j}\) via $A$ and $\Delta_i$.

A distinct divergence between Mamba's attention mechanism and traditional self-attention lies in the latter's utilization of a per-row softmax function. It is essential to recognize that various attention models have either omitted the softmax~\cite{lu2021soft} or substituted it with elementwise neural activations~\cite{hua2022transformer,wortsman2023replacing,ma2022mega,zimerman2023converting}, achieving comparable outcomes to the original framework.

{The softmax operator is known to lead to oversmoothing~\cite{ali2023centered,wang2022anti}. As we show in the Appendix~\ref{app:oversmoothing}, Mamba attention layers have lower inter-token smoothing than Transformer attention.}

\subsection{Application to Attention Rollout\label{sec:AttnMatForXAI}}

{
As our class-agnostic explainability teqchniqe for Mamba models, we built our method on top of the Attention-Rollout~\cite{abnar2020quantifying} method. For simplicity, we assume that we are dealing with a vision mamba model, which operates on sequences of size $L+1$, where $L$ is the sequence length obtained from the $\sqrt{L} \times \sqrt{L}$ image patches, with a classification (CLS) token appended to the end of the sequence.

To do so, for each sample, we first extract the hidden attention matrix $\tilde{\alpha}^{\lambda,d}$ for any channel $d \in [D]$ and layer $\lambda \in [\Lambda]$ according to the formulation in section~\ref{sec:S6asAttn} (Eq.~\ref{eq:MAMbaASmatmul}), such that $\tilde{\alpha}^{\lambda,d} \in \mathbb{R}^{(L+1)\times (L+1)}$ 

Attention-Rollout is then applied as follows:
\begin{equation}\label{eq:rolloutPerchannel}
    \forall \lambda \in [\Lambda]:\quad   \tilde{\mathbb{\alpha}}^{\lambda} = \mathbb{I}_{L+1} + \mathop{\mathbb{E}}_{d \in [D]}(\tilde{\alpha}^{\lambda,d}), \quad 
    {\tilde{\alpha}}^{\lambda} \in \mathbb{R}^{(L+1) \times (L+1)}
\end{equation}
where 
$\mathbb{I}_{L+1} \in \mathbb{R}^{(L+1) \times (L+1)}$ is an identity matrix utilized to incorporate the influence of skip connections along the layers.

Now, the per-layer global attention matrices $\tilde{\mathbb{\alpha}}^{\lambda}$ for all $\lambda \in [\Lambda]$ are aggregated into the final map $\rho$ by: 
\begin{equation}\label{eq:rolloutPerlayer}
    \rho = \Pi_{\lambda=1}^{\Lambda} \tilde{\alpha}^{\lambda}, \quad \rho \in \mathbb{R}^{(L+1) \times (L+1)}  
\end{equation}

Note that each row of $\rho$ corresponds to a relevance map for each token, given the other tokens. In the context of this study, which concentrates on classification models, our attention analysis directs attention exclusively to the CLS token. Thus, we derive the final relevance map from the row associated with the CLS token in the output matrix, denoted by $\rho_{\text{CLS}} \in \mathbb{R}^{L}$, which contains the relevance scores evaluating each token's influence on the classification token. Finally, to obtain the final explanation heatmap we reshape $\rho_{\text{CLS}} \in \mathbb{R}^{L}$ to $\sqrt{L}\times\sqrt{L}$ and upsample it back to the size of the original image using bilinear interpolation. 

Although Mamba models are causal by definition, resulting in causal hidden attention matrices, our method can be extended to a bidirectional setting in a straightforward manner. This adaptation involves modifying Eq.~\ref{eq:rolloutPerchannel} so that $\tilde{\mathbb{\alpha}}^{\lambda,d}$ becomes the outcome of summing the (two) per-direction matrices of the $\lambda$-layer and the $d$-channel. 

}

\subsection{Application to Attention-based Attribution\label{sec:AttrForXAI}}
{As our class-specific explainability technique for Mamba models, we have tailored the Transformer-Attribution~\cite{chefer2021transformer} explainability method, which is specifically designed for transformers, to suit Mamba models. This method relies on a combination of LRP scores and attention gradients to generate the relevance scores. Since each Mamba block includes several peripheral layers that are not included in transformers, such as Conv1D, additional gating mechanisms, and multiple linear projection layers, a robust mechanism must be designed carefully. For simplicity, we focus on vision Mamba, with a grid of $\sqrt{L}$ patches in each row and column, 
as in Sec.~\ref{sec:AttnMatForXAI}.

The Transformer-Attribution method encompasses two stages: (i) generating a relevance map for each attention layer, followed by (ii) the aggregation of these relevance maps across all layers, using the aggregation rule specified in~\ref{eq:rolloutPerlayer}, to produce the final map $\rho$. 

The difference from the attention rollout method therefore lies in how step (i) is applied to each Mamba layer $\lambda \in [ \Lambda] $. For the $\hat{h} \in [H]$ attention head at layer $\lambda$, the transformer method~\cite{chefer2021transformer} computes the following two maps: (1) LRP \cite{bach2015pixel} relevance scores map $R^{\lambda,\hat{h}}$, and (2) the gradients $\nabla {\tilde{\alpha}}^{\lambda,\hat{h}}$ with respect to a target class of interest. Then, these two are fused by a Hadamard product: 
\begin{equation}\label{eq:transformerAttOneLayer}
    \beta^{\lambda} = \mathbb{I}_L +\mathop{\mathbb{E}}_{\hat{h} \in [\hat{H}]}(\nabla {\alpha}^{\lambda,\hat{h}} \odot R^{\lambda,\hat{h}})^+, \quad \mathbb{I}_{L+1} \in \mathbb{R}^{(L+1) \times (L+1)}
\end{equation}

\begin{figure*}[t]
 \centering
  \begin{minipage}[c]{0.549\linewidth}
\includegraphics[width=0.9945\textwidth]{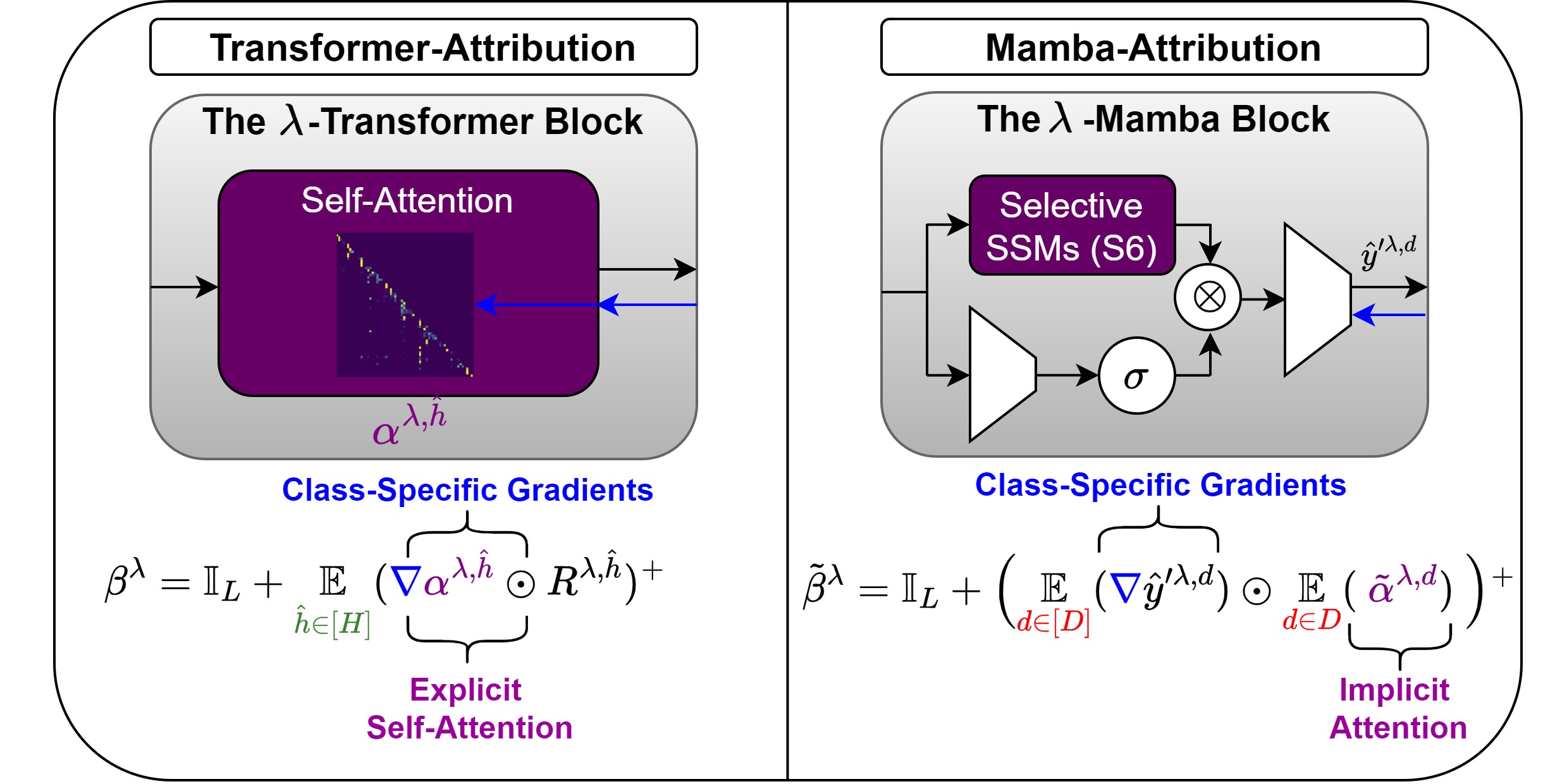}
\caption{Comperative Visualization of Transformer-Attribution and our Mamba-Attribution, both class specific methods.}
\label{fig:MambaAttribution}
 \end{minipage}%
\hfill
\begin{minipage}[c]{0.39\linewidth}

\includegraphics[width=0.99945\textwidth]{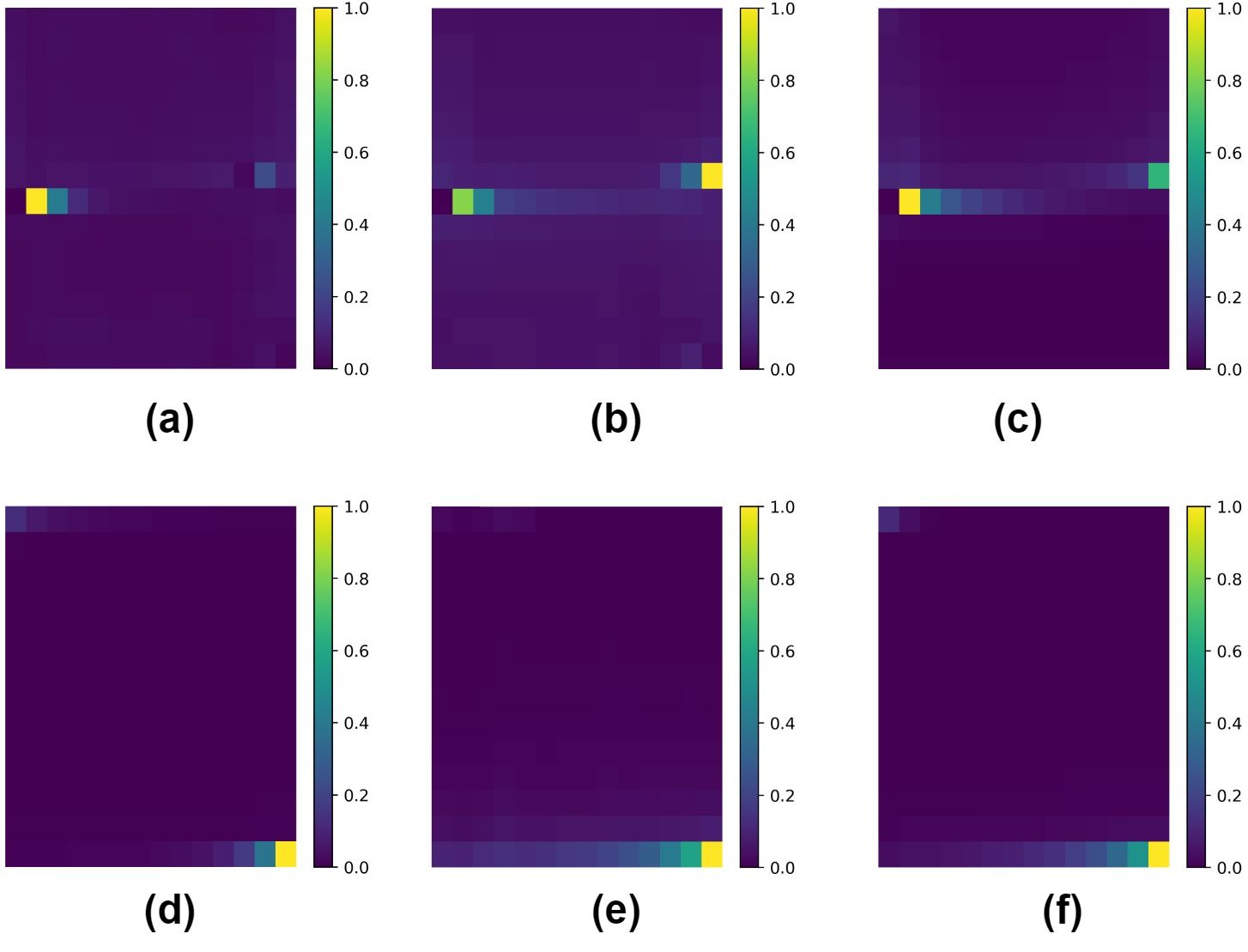}
\vspace{-.6cm}
\caption{Average attention maps for CLS token in the middle (a,b,c) and as the first (d,e,f).}
\label{fig:CLStoken}
\end{minipage}
\end{figure*}

Our method, \textbf{Mamba-Attribution}, depicted in Fig.~\ref{fig:MambaAttribution}, deviates from this method by modifying Eq.~\ref{eq:transformerAttOneLayer} in the following aspects: (i) Instead of computing the gradients on the per-head attention matrices $\nabla {\alpha}^{\lambda,\hat{h}}$, we compute the gradients of $\nabla \hat{y}'^{\lambda,d}$. The motivation for these modifications is to exploit the gradients of both the S6 mixer and the gating mechanism in Eq.~\ref{eq:output_block} (left), to obtain strong class-specific maps. (ii) We simply replace $R^{\lambda,\hat{h}}$ with the attention matrices $\tilde{\alpha}^{\lambda,d}$ at layer $\lambda$ and channel $d$, since we empirically observe that those attention matrices produce better relevance maps. Both of these modifications are manifested by the following form, which defines our method:

\begin{equation}\label{eq:transformerAttOneLayerVB}
    \tilde{\beta}^{\lambda} = \mathbb{I}_L + \Big{(}\mathop{\mathbb{E}}_{d \in D}(\nabla\hat{y}'^{\lambda,d}) \odot \mathop{\mathbb{E}}_{d \in D} (\tilde{\alpha}^{\lambda,d})\Big{)}{}^+
\end{equation}
}

\section{Experiments}
In this section, we present an in-depth analysis of the hidden attention mechanism embedded within Mamba models, focusing on its semantic diversity and applicability in explainable AI frameworks. %
We start by visualizing the hidden attention matrices for both NLP and vision models in Sec.~\ref{subsec:attnMatVis}, followed by assessing our explainable AI techniques empirically, via perturbation and segmentation tests in Sec.~\ref{sec:metrics}.

\subsection{Visualization of Attention Matrices}\label{subsec:attnMatVis}

The Visual Mamba (ViM) comes in two versions: in one, the CLS token is last 
and in the other, the CLS token is placed in the middle. Fig.~\ref{fig:CLStoken} shows how this positioning influences the impact of the patches on the CLS, 
by averaging over the entire test set. Evidently, the patches near the CLS token are more influential. This phenomenon may suggest that a better strategy is to have a non-spatial/global CLS token~\cite{farooq2021global,hatamizadeh2023global}.


Fig.~\ref{fig:hiddenAttn} 
compares the attention matrices in Mamba and Transformer on both vision and NLP tasks. For clearer visualization, we apply the Softmax function to each row of the attention matrices obtained from transformers and perform min-max normalization on the absolute values of the Mamba matrices.  In all cases, we limit our focus to the first 64 tokens. In vision, we compare Vision-Mamba (ViM) and ViT (DeiT), for models of a tiny size, trained on ImageNet-1K. The attention maps are extracted using examples from the test set. Each Mamba attention matrix is obtained by combining the two maps of the bidirectional channel. In NLP, we compare attention matrices extracted from Mamba (130m) and Transformer (Pythia-160m~\cite{biderman2023pythia}) language models, trained on the Pile~\cite{gao2020pile} dataset for next token prediction. The attention maps are extracted using examples from the Lambada dataset (preprocessed by OpenAI). 

\begin{figure*}[t]
\centering

\begin{tabular}{cccc} 
  \multicolumn{2}{c}{Vision}   & \multicolumn{2}{c}{NLP}\\
  \cmidrule(lr){1-2}
  \cmidrule(lr){3-4}
  Transformer & Mamba & Transformer & Mamba \\
    \includegraphics[width=0.252\textwidth]{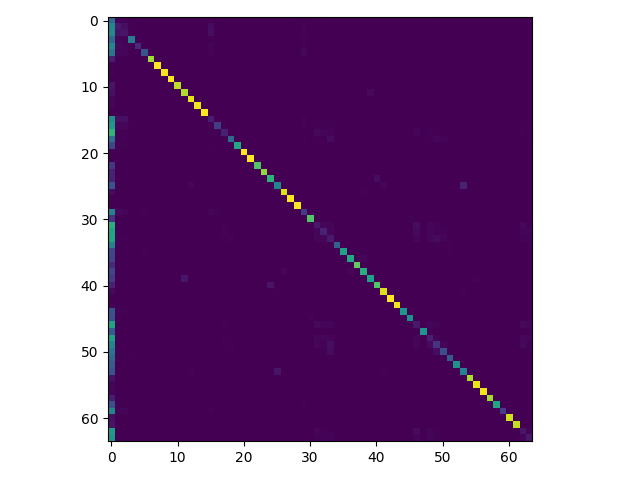} & \includegraphics[width=0.252\textwidth]{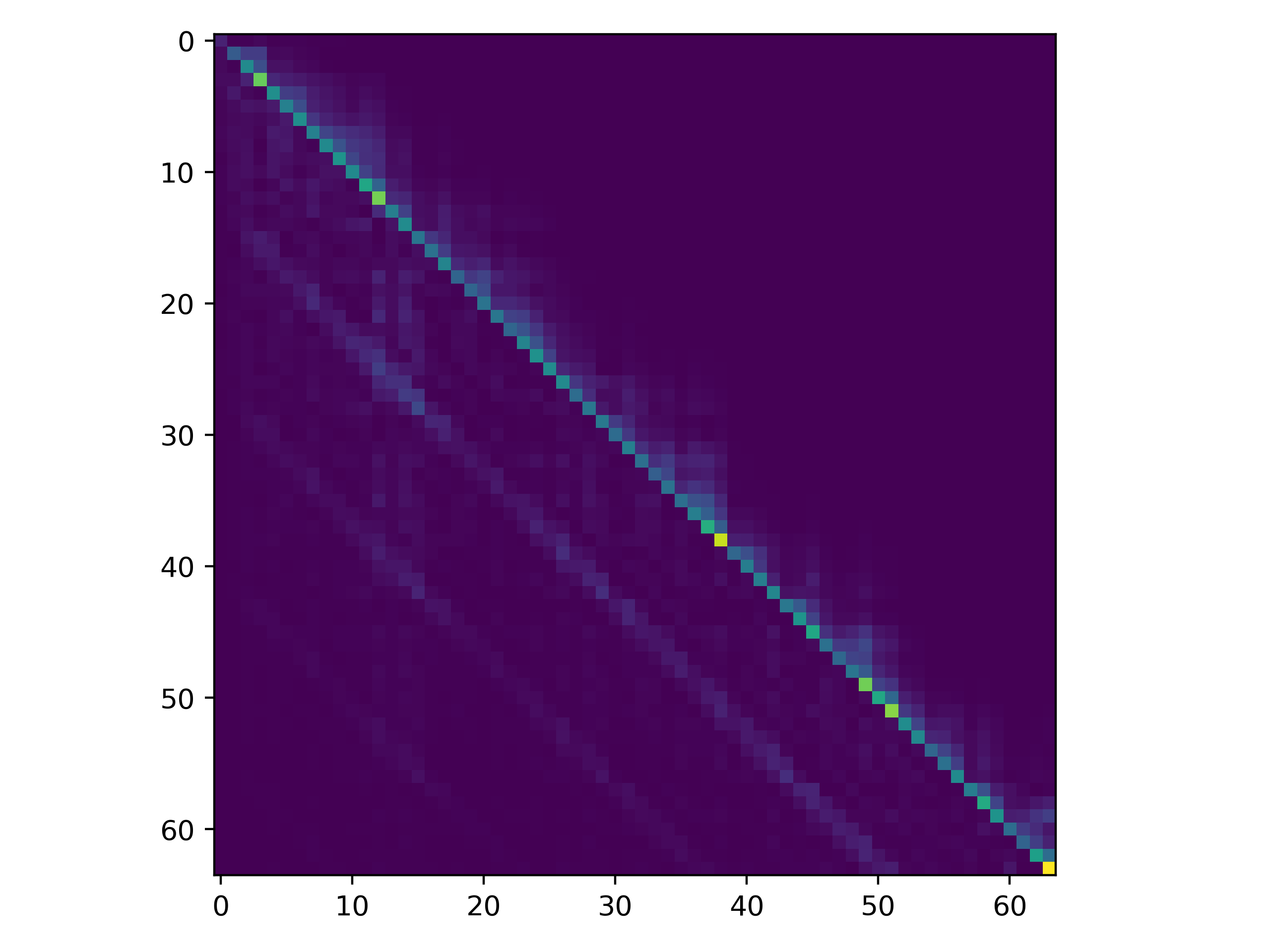} &
    \includegraphics[width=0.252\textwidth]{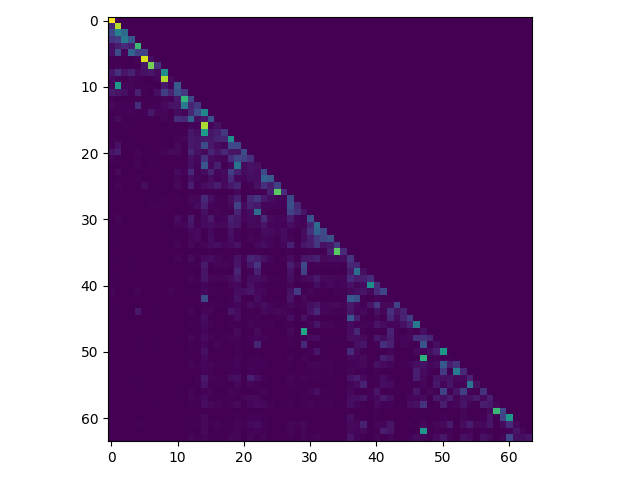} & \includegraphics[width=0.252\textwidth]{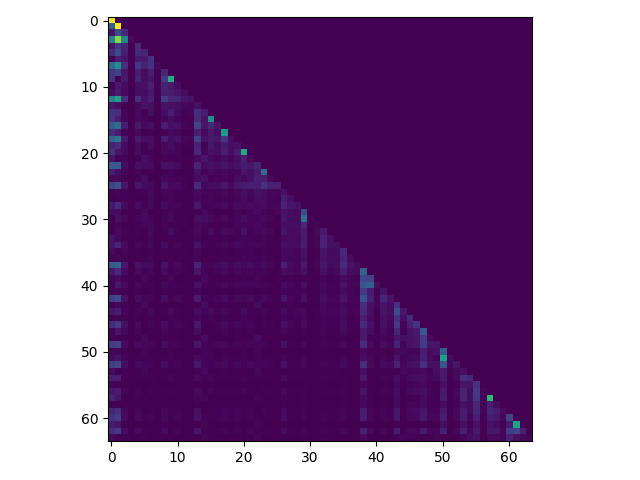}
 \\
  \includegraphics[width=0.252\textwidth]{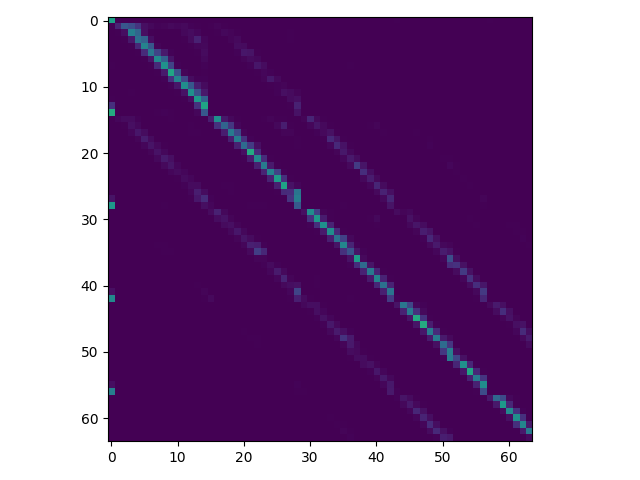} & \includegraphics[width=0.252\textwidth]{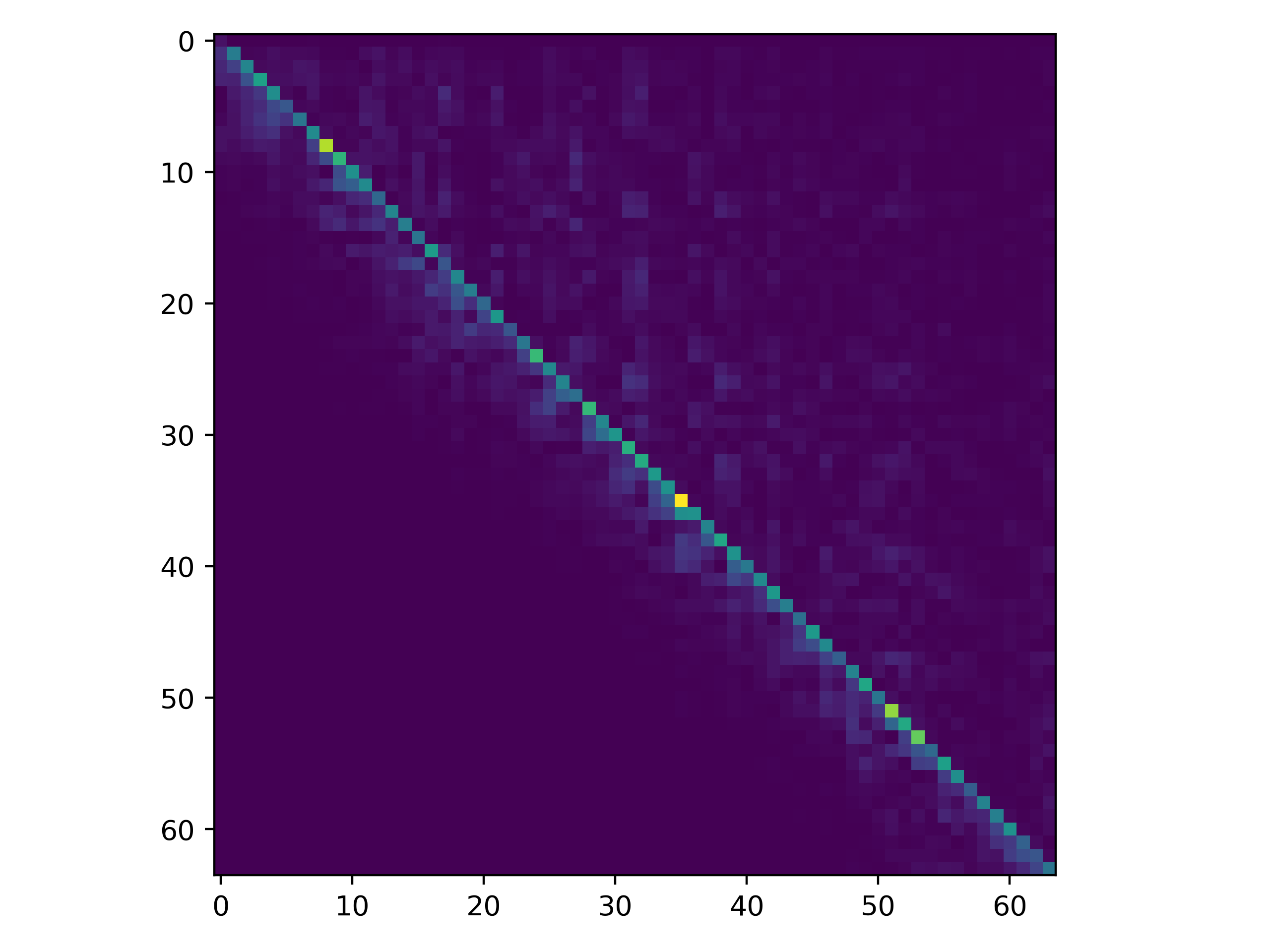} &
    \includegraphics[width=0.252\textwidth]{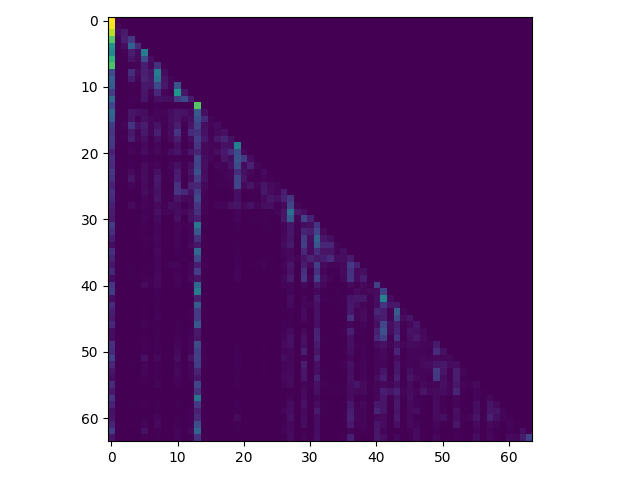} & \includegraphics[width=0.252\textwidth]{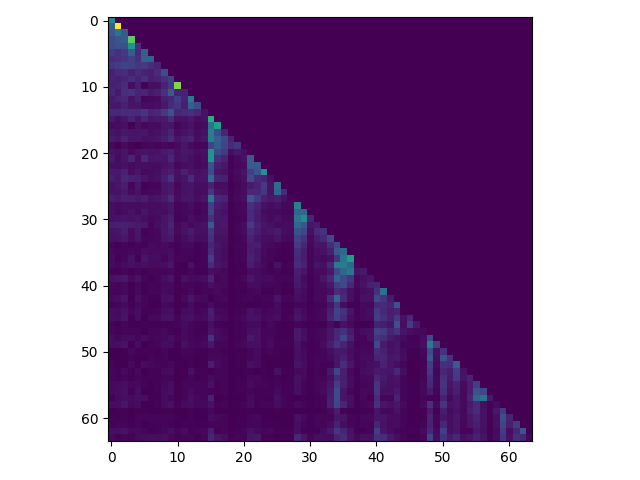}
 \\
  \includegraphics[width=0.252\textwidth]{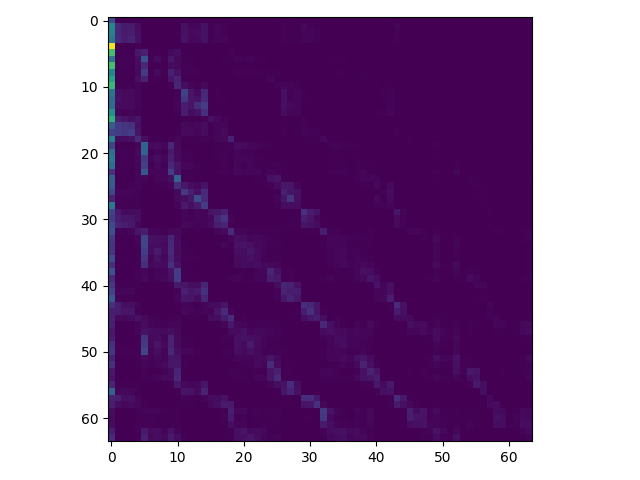} & \includegraphics[width=0.252\textwidth]{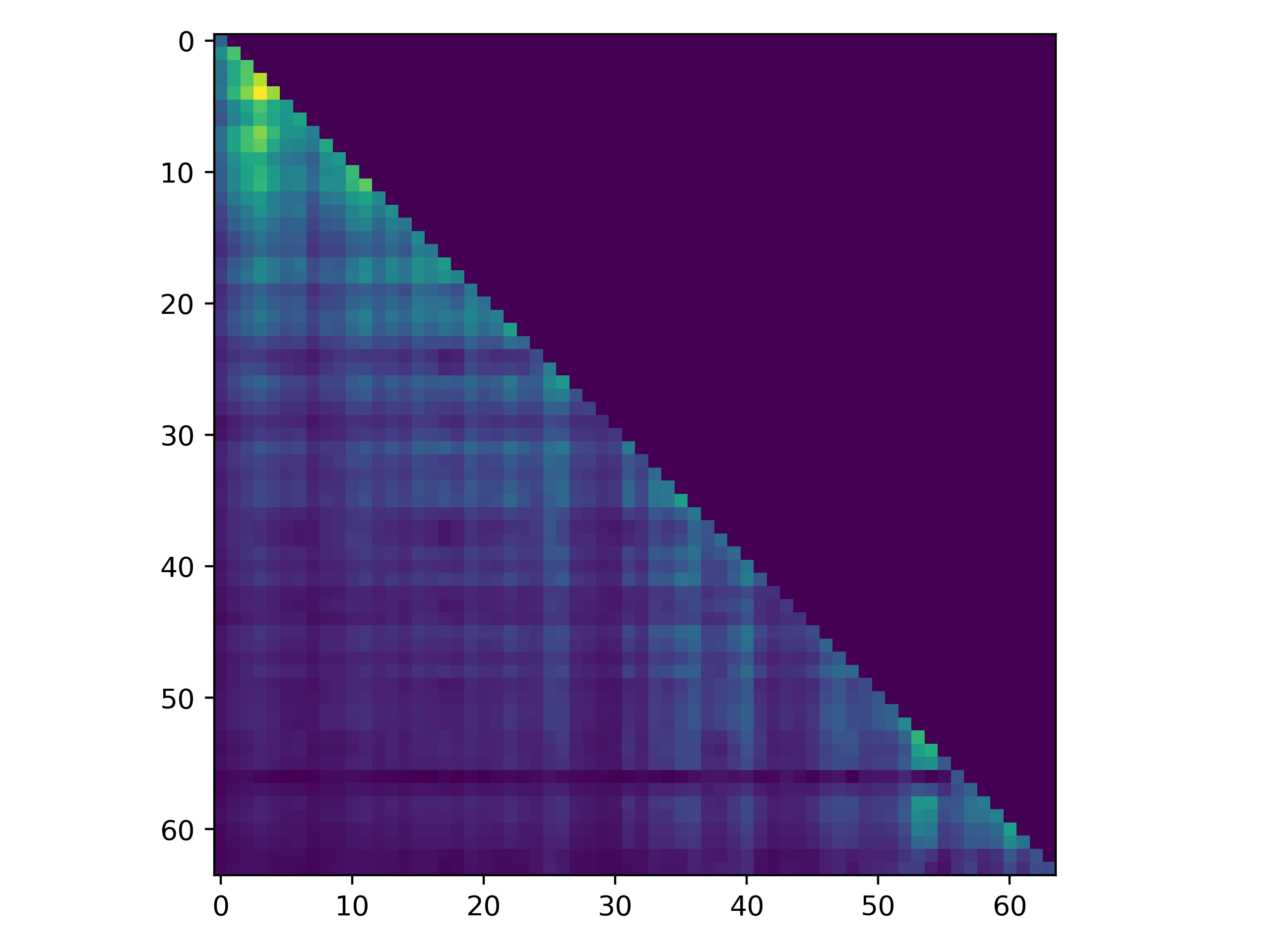} &
    \includegraphics[width=0.252\textwidth]{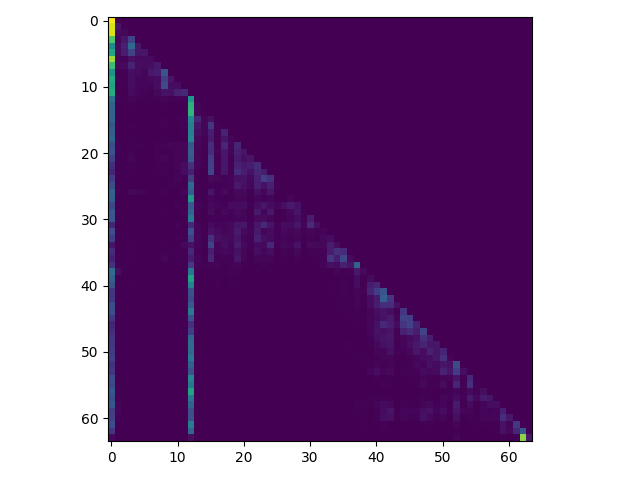} & \includegraphics[width=0.252\textwidth]{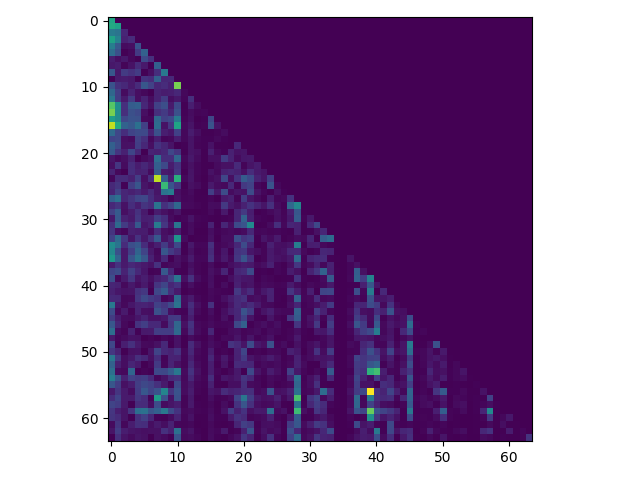}
 \\
 \end{tabular}
 \caption{\textbf{Hidden Attention Matrices:} 
 Attention matrices in vision and NLP Models.Each row represents a different layer within the models, showcasing the evolution of the attention matrices at 25\% (top), 
    50\%, and 75\% (bottom) of the layer depth.}\label{fig:hiddenAttn}
\end{figure*}



As can be seen, the hidden attention matrices of Mamba appear to be similar to the attention matrices extracted from transformers 
In both models, 
the dependencies between distant tokens are captured in the deeper layers of the model, as depicted in the lower rows. 

Some of the attention matrices demonstrate the ability of selective SSM models and transformers to focus on parts of the input. In those cases, instead of the diagonal patterns, some columns seem to miss the diagonal element and the attention is more diffused (recall that we normalized the Mamba attention maps for visualization purposes. In practice, these columns have little activity).

 
Evidently, both the Mamba attention matrices and the transformer attention matrices possess similar properties and depict the two-dimensional structure within the data as bands with an offset of $\sqrt{L}$.

\begin{figure*}[t]
\centering
\begin{tabular}{ccccccc}
\includegraphics[width=0.14\linewidth, height=2cm]{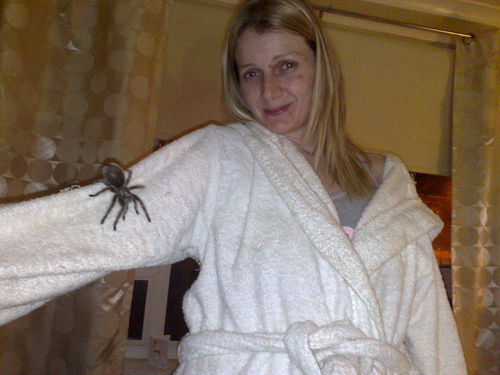} &
\includegraphics[width=0.14\linewidth, height=2cm]{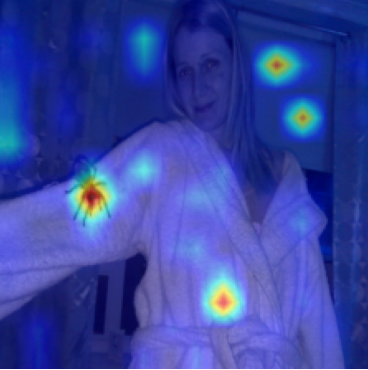} &
\includegraphics[width=0.14\linewidth, height=2cm]{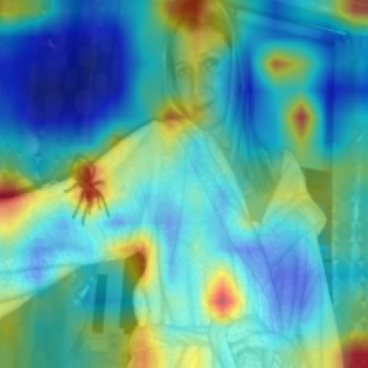} &
\includegraphics[width=0.14\linewidth, height=2cm]{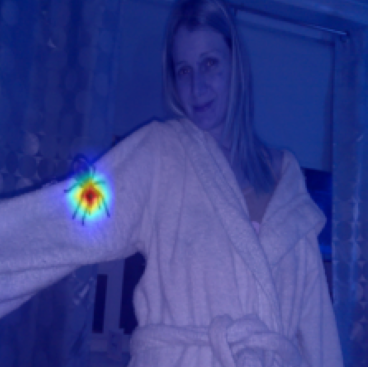} &
\includegraphics[width=0.14\linewidth, height=2cm]{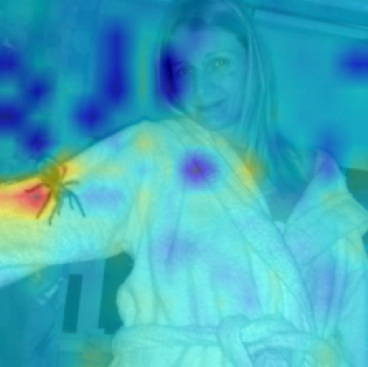} &
\includegraphics[width=0.14\linewidth, height=2cm]{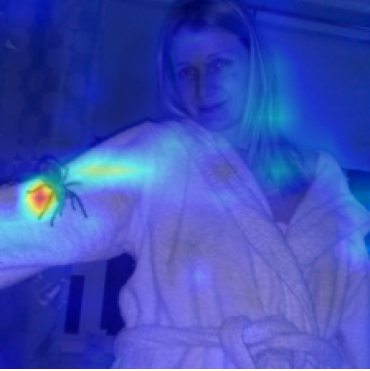} &
\includegraphics[width=0.14\linewidth, height=2cm]{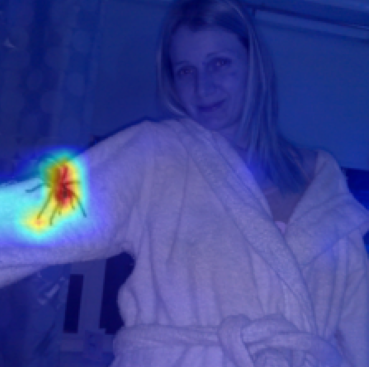} \\
\includegraphics[width=0.14\linewidth, height=2cm]{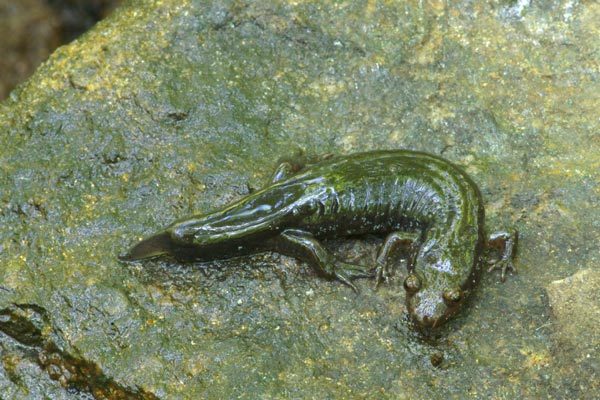} &
\includegraphics[width=0.14\linewidth, height=2cm]{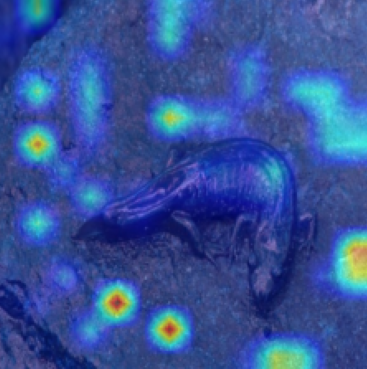} &
\includegraphics[width=0.14\linewidth, height=2cm]{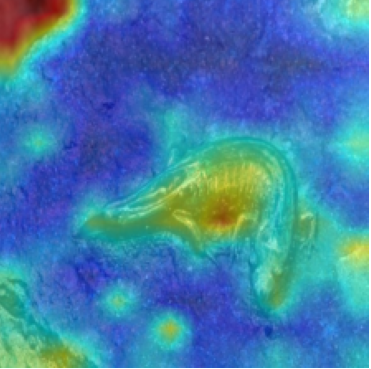} &
\includegraphics[width=0.14\linewidth, height=2cm]{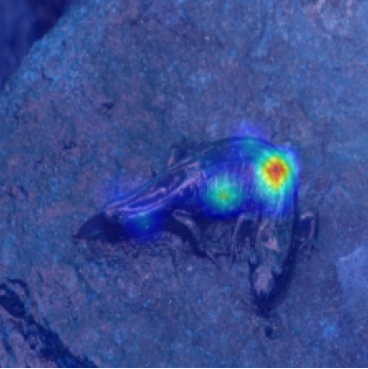} &
\includegraphics[width=0.14\linewidth, height=2cm]{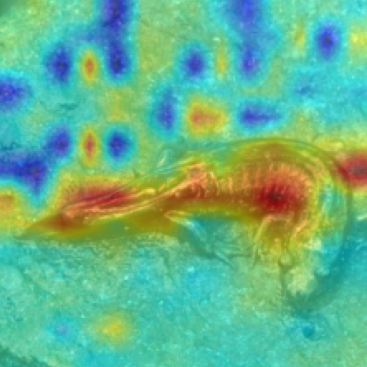} &
\includegraphics[width=0.14\linewidth, height=2cm]{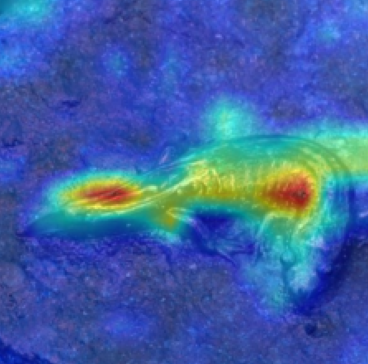} &
\includegraphics[width=0.14\linewidth, height=2cm]{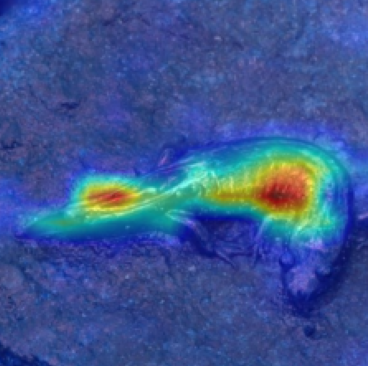} \\
\includegraphics[width=0.14\linewidth, height=2cm]{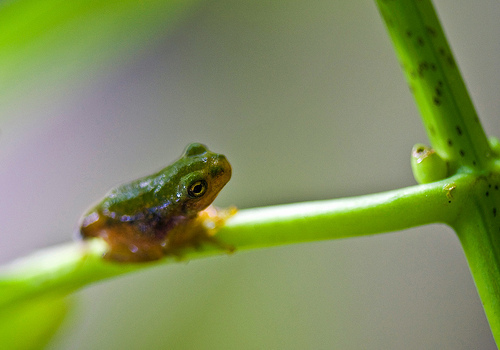} &
\includegraphics[width=0.14\linewidth, height=2cm]{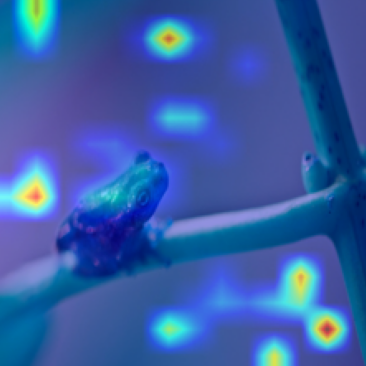} &
\includegraphics[width=0.14\linewidth, height=2cm]{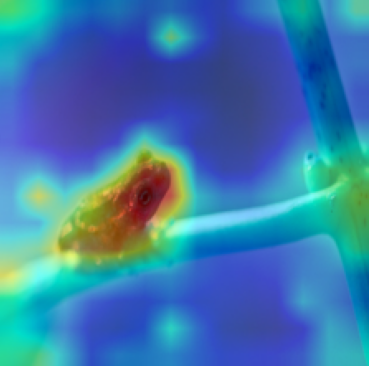} &
\includegraphics[width=0.14\linewidth, height=2cm]{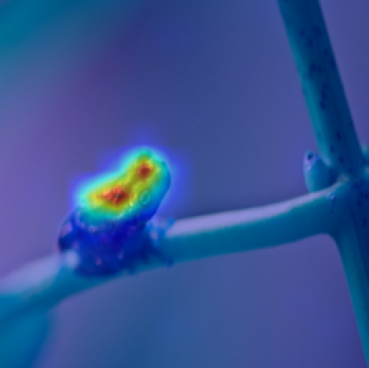} &
\includegraphics[width=0.14\linewidth, height=2cm]{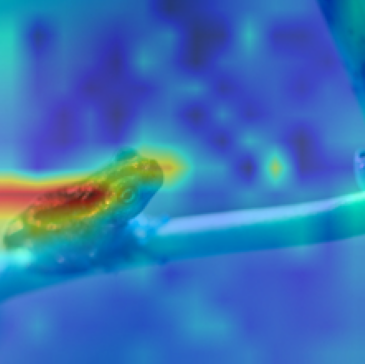} &
\includegraphics[width=0.14\linewidth, height=2cm]{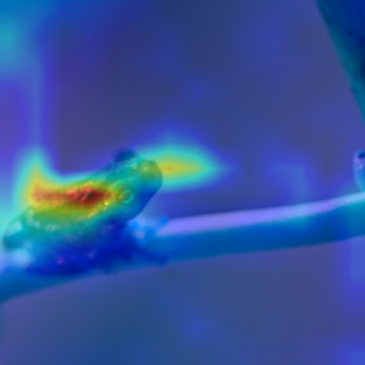} &
\includegraphics[width=0.14\linewidth, height=2cm]{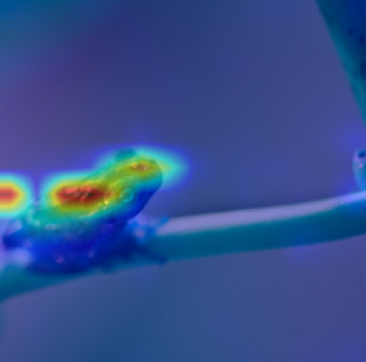} \\
\includegraphics[width=0.14\linewidth, height=2cm]{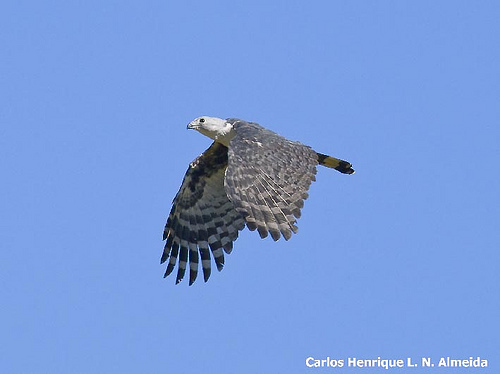} &
\includegraphics[width=0.14\linewidth, height=2cm]{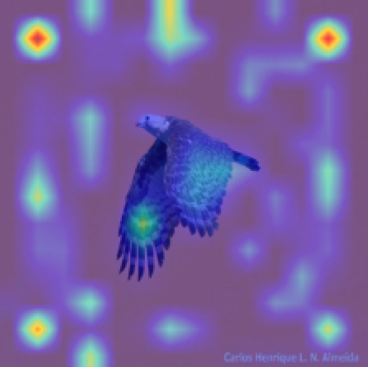} &
\includegraphics[width=0.14\linewidth, height=2cm]{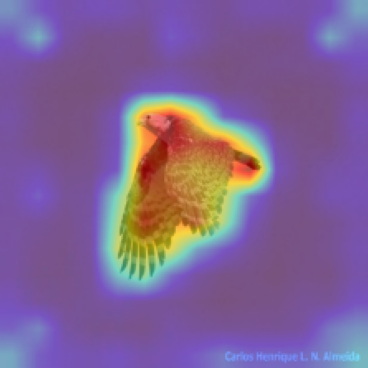} &
\includegraphics[width=0.14\linewidth, height=2cm]{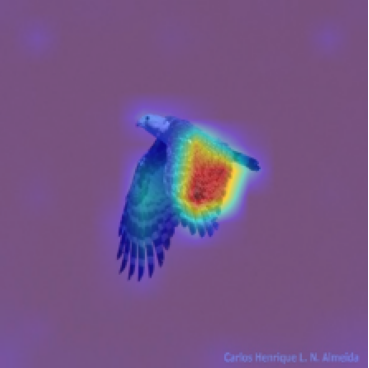} &
\includegraphics[width=0.14\linewidth, height=2cm]{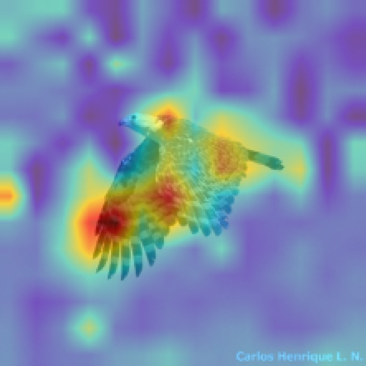} &
\includegraphics[width=0.14\linewidth, height=2cm]{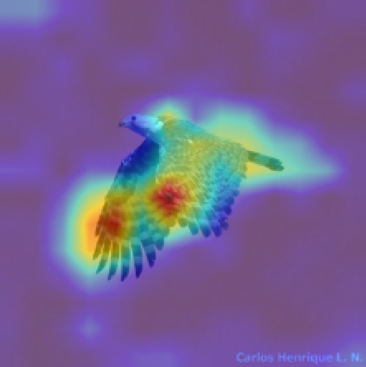} &
\includegraphics[width=0.14\linewidth, height=2cm]{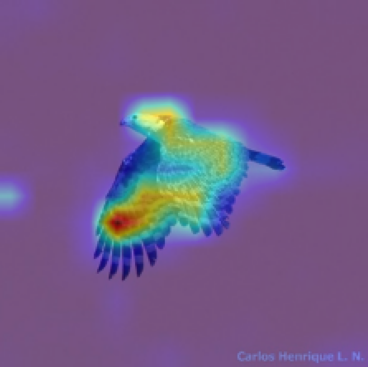} \\
(a) & (b) & (c) & (d) & (e) & (f) & (g) \\
\end{tabular}

  \caption{Qualitative results for the different explanation methods for the ViT-small and the Mamba-small models. (a) the original image, (b) the aggregated Raw-Attention of ViT-Small, (c) Attention Rollout for ViT-Small, (d) Transformer-Attribution for ViT-Small, (e) the Raw-Attention of Mamba-Small, (f) Attention-Rollout of Mamba-Small and (g) the Mamba-Attribution method for the Mamba-Small model.}
  \label{fig:qualitative_cmp}
\end{figure*}

\subsection{Explainability Metrics}
\label{sec:metrics}
The explainable AI experiments include three types of explainability methods: (1) Raw-Attention, which employs raw attention scores as relevancies. Our findings indicate that averaging the attention maps across layers yields optimal results. (2) Attn-Rollout~\cite{abnar2020quantifying} for Transformers, and its Mamba version, as depicted in Sec.~\ref{sec:AttnMatForXAI}. Finally, (3) The Transformer Attribution of Chefer et al.~\cite{chefer2021generic} and its Mamba Attribution counterpart, detailed in Sec.~\ref{sec:AttrForXAI}.

Fig.~\ref{fig:qualitative_cmp} depicts the results of the six attribution methods on typical samples from the ImageNet test set. As can be seen, the Mamba-based heatmaps are often more complete than their transformer-based counterparts. The raw attention of Mamba stands out from the other five heatmaps, since it depicts activity across the entire image. However, the relevant object is highlighted.

Next, we apply explainability evaluation metrics. These metrics allow one to compare different explainability methods that are applied to the same model. Applying them to compare different models is not meant to say that model X is more explainable than model Y. The main purpose is to show that the attention maps of Mamba we introduce are as useful as the attention maps of Transformers in terms of providing explainability. A secondary aim is to validate the feasibility of potential use for weakly supervised downstream tasks that require spatial location.
Perturbation

\smallskip
\noindent{\bf Perturbation Tests\quad}
In this evaluation framework, we employ an input perturbation scheme to assess the efficacy of various explanation methods, following the approach outlined by ~\cite{chefer2021transformer,chefer2021generic}. 

These experiments are conducted under two distinct settings. 
In the positive perturbation scenario, a quality explanation involves an ordered list of pixels, arranged most-to-least relevant. Consequently, when gradually masking out the pixels of the input image, starting from the highest relevance to the lowest, and measuring the mean top-1 accuracy of the network, one anticipates a notable decrease in performance. 

Conversely, in the negative perturbation setup, a robust explanation is expected to uphold the accuracy of the model while systematically removing pixels, starting from the lowest relevance to the highest. 

In both cases, the evaluation metrics consider the area-under-curve (AUC) focusing on the erasure of 10\% to 90\% of the pixels. 

{The results of the perturbations are presented in Tab.~\ref{tab:perturbations_transposed}, depicting the performance of different explanation methods under both positive and negative perturbation scenarios across the two models. In the positive perturbation scenario, where lower AUC values are indicative of better performance, we notice that for Raw-Attention, Mamba shows a better AUC  compared to the Vision Transformer (ViT). For the Attn-Rollout method, Mamba outperforms the ViT, while the latter shows a better AUC under the Attribution method. In the negative perturbation scenario, where higher AUC values are better, the Transformer-based methods consistently outperform Mammba across all three methods. The tendency for lower AUC in both positive (where it is desirable) and negative perturbation (where it is undesirable) may indicate that the Mamba model is more sensitive to blacking out patches, and it would be interesting to add experiments in which the patches are blurred instead~\cite{fong2017interpretable}.}
For perturbation tests in the NLP domain, please refer to the Appendices~\ref{app:nlpexpir} and~\ref{app:nlp_qual}.

\smallskip
\noindent{\bf Segmentation Tests \quad}
It is expected that an effective explainability method would produce reasonable foreground segmentation maps. This is assessed for ImageNet classifiers by comparing the obtained heatmap against the ground truth segmentation maps available in the ImageNet-Segmentation dataset~\cite{guillaumin2014ImageNet}. 

Evaluation is conducted based on pixel accuracy, mean-intersection-over-union (mIoU) and mean average precision (mAP) metrics, aligning with established benchmarks in the literature for explainability~\cite{chefer2021generic,chefer2021transformer,nam2020relative,gur2021visualization}.

The results are outlined in Tab.~\ref{tab:segmentation_transposed}. For Raw-Attention, Mamba demonstrates significantly higher pixel accuracy and mean Intersection over Union compared to Vision Transformer, while the latter performs better in mean Average Precision. Under the Attn-Rollout method, Mamba outperforms Vision Transformer in mean Average Precision, pixel accuracy and mean Intersection over Union. Finally, Transformer-Attribution consistently surpasses Mamba-Attribution, achieving the highest scores in pixel accuracy, mean Average Precision, and mean Intersection over Union, respectively. 

These results underscore the potential of Mamba's attention mechanism as approaching and sometimes surpassing the interoperability level of Transformer models, especially when the attention maps are taken as is. It also highlights the applicability of Mamba models for downstream tasks such as weakly supervised segmentation. It seems, however, that the Mamba-based attribution model, which is modeled closely after the transformer method of Chefer et al.~\cite{chefer2021transformer} may benefit from further adjustments.

\begin{table*}[t]
    \caption{Positive and Negative perturbation AUC results (percentages) for the predicted class 
    on the ImageNet validation set. For positive perturbation lower is better, and for negative perturbation higher is better.}
    \label{tab:perturbations_transposed}
    \centering
    \small
    \begin{tabular*}{\linewidth}{@{\extracolsep{\fill}}lcccc}
        \toprule
        & \multicolumn{2}{c}{Positive Perturbation} & \multicolumn{2}{c}{Negative Perturbation} \\
        \cmidrule{2-3} \cmidrule{4-5}
        & Mamba & Transformer & Mamba & Transformer \\
        \midrule
        Raw-Attention & 17.268 & 20.687 & 34.025 & 40.766 \\
        Attn-Rollout & 18.806 & 20.594 & 41.864 & 43.525 \\
        Attribution & 16.619 & \textbf{15.351} & 39.632 & \textbf{48.089} \\
        \midrule
    \end{tabular*}
\smallskip
\caption{Segmentation performance on the ImageNet-Segmentation~\cite{guillaumin2014ImageNet} dataset (percent). Higher is better. The upper part depicts the results for Vision Mamba-small while the lower part contains the results for Vision Transformer-Small}
\label{tab:segmentation_transposed}
\centering
\begin{tabular*}{\linewidth}{@{\extracolsep{\fill}}llccc}
\toprule
\small
Model & Method & pixel accuracy & mAP & mIoU \\
\midrule
{Transformer} & Raw-Attention &59.69  & \textbf{77.25} & 36.94\\
{Mamba} & Raw-Attention & \textbf{67.64} & 74.88 & \textbf{45.09} \\
\midrule
 {Transformer}& Attn-Rollout~\cite{abnar2020quantifying} & 66.84 &  80.34& 47.85 \\
 {Mamba}& Attn-Rollout (Sec.~\ref{sec:AttnMatForXAI}) & \textbf{71.01} & \textbf{80.78} & \textbf{51.51} \\
\midrule
 {Transformer}& Transformer-Attr~\cite{chefer2021transformer} & \textbf{79.26} & \textbf{84.85} & \textbf{60.63} \\
 {Mamba}& Mamba-Attr (Sec.~\ref{sec:AttrForXAI}) & 74.72 & 81.70 & 54.24 \\
\bottomrule
\end{tabular*}
\end{table*}

\section{Discussion: The Evolution of Attention in SSMs}


A natural question to ask is whether the attention perspective we exposed is unique to Selective SSM (the building block of Mamba), separating it from other SSMs. The answer is that Selective SSM, similar to transformers, contains a type of layer we call data-dependent non-diagonal mixer, which previous layers do not. 

In their seminal work, Poli et al.~\cite{poli2023hyena} claim that a crucial aspect of transformers is the existence of an \emph{expressive, data-controlled linear operator}. Here, we focus on a more specific component, which is \emph {an expressive data-controlled linear non-diagonal mixer operator}. This distinguishes between elementwise operators that act on the data associated with specific tokens (including MLP, gating mechanisms, and GLU activations~\cite{shazeer2020glu}) and mixer operations that pool information from multiple tokens. 

The mixer components can further be divided into fixed, e.g., using pooling operators with fixed structure and coefficients, or data-dependent, in which the interactions between tokens are controlled by their input-dependent representations, e.g., self-attention. In appendix~\ref{app:mixing} we prove the following result, which sheds light on the gradual evolution of attention in SSM models.
\begin{theorem}
(i) S4~\cite{gu2021efficiently}, DSS~\cite{dss}, S5~\cite{smith2022simplified} have fixed mixing elements. (ii) GSS~\cite{gss},and Hyena~\cite{poli2023hyena} have fixed mixing elements with diagonal data-control mechanism. (iii) Selective SSM have data-controlled non-diagonal mixers.
\end{theorem}

Transformers are recognized for their superior in-context learning (ICL) capabilities, where the model adapts its function according to the input provided~\cite{brown2020language}. Empirical evidence has demonstrated that Mamba models are the first SSMs to exhibit ICL capabilities on par with those of transformers~\cite{MAMABICL1,MAMABAICL2}. Based on the intuition that the capacity to focus on specific inputs is required for ICL, we hypothesize that the presence of data-controlled non-diagonal mixers in both transformers and Mamba models is crucial for achieving a high level of ICL. 

A question then arises: which model is more expressive, transformers or selective SSM? While previous work has shown that Transformers are more expressive than traditional state-space layers~\cite{zimerman2023long}, we show in Appendix~\ref{app:express} that the situation is reversed for selective SSMs, as follows:
\begin{theorem}
One channel of the selective state-space layer can express all functions that a single transformer head can express. Conversely, a single Transformer layer cannot express all functions that a single selective SSM layer can. 
\end{theorem}

\section{Conclusions}
In this work, we have established a significant link between Mamba and self-attention layers, illustrating that the Mamba layer can be reformulated as an implicit form of causal self-attention mechanism. This links the highly effective Mamba layers directly with the transformer layers. 
 
The parallel perspective plays a crucial role in efficient training and the recurrent perspective is essential for effective causal generation. The attention perspective plays a role in understanding the inner representation of the Mamba model. While ``Attention is not Explanation'' \cite{jain2019attention}, attention layers have been widely used for transformer explainability. By leveraging the obtained attention matrices, we introduce the first (as far as we can ascertain) explainability techniques for Mamba models, for both task-specific and task-agnostic regimes. This contribution equips the research community with novel tools for examining the performance, fairness, robustness, and weaknesses of Mamba models, thereby paving the way for future improvements, and it also enables weakly supervised downstream tasks. 
Looking ahead, we plan to delve into the relationships between Mamba, Self-Attention and other recent layers, such as RWKV~\cite{peng2023rwkv}, Retention~\cite{sun2307retentive} and Hyena~\cite{poli2023hyena}, and develop XAI methods for LLMs relying on these layers and their corresponding vision variants~\cite{fan2023rmt,zimerman2023multi}.


%
%

\section{Acknowledgments}
This work was supported by a grant from the Tel Aviv University Center for AI and Data Science (TAD). This research was also supported by the Ministry of Innovation, Science \& Technology ,Israel (1001576154) and the Michael J. Fox Foundation (MJFF-022407). The contribution of the first author is part of a  PhD thesis research conducted at Tel Aviv University.

\bibliographystyle{splncs04}
\bibliography{main}
\newpage
\appendix
\section{Oversmoothing}
\setcounter{equation}{23}
\label{app:oversmoothing}

Oversmoothing poses a significant challenge for Transformer-based models~\cite{ru2023token,ali2023centered,dovonon2024setting,nguyen2024mitigating,wang2022anti,gong2021vision,ali2023centered}, 
affecting their ability to accurately capture and represent intricate features and relationships. In Transformer-based models, oversmoothing impacts tasks such as natural language processing~\cite{chen-etal-2023-alleviating,kulikov-etal-2022-characterizing} and vision~\cite{ru2023token,ali2023centered,gong2021vision}, where it blurs fine-grained details essential for understanding and localization~\cite{ru2023token,ali2023centered}.
\begin{figure}[]
\centering
\includegraphics[width=1.0\textwidth]{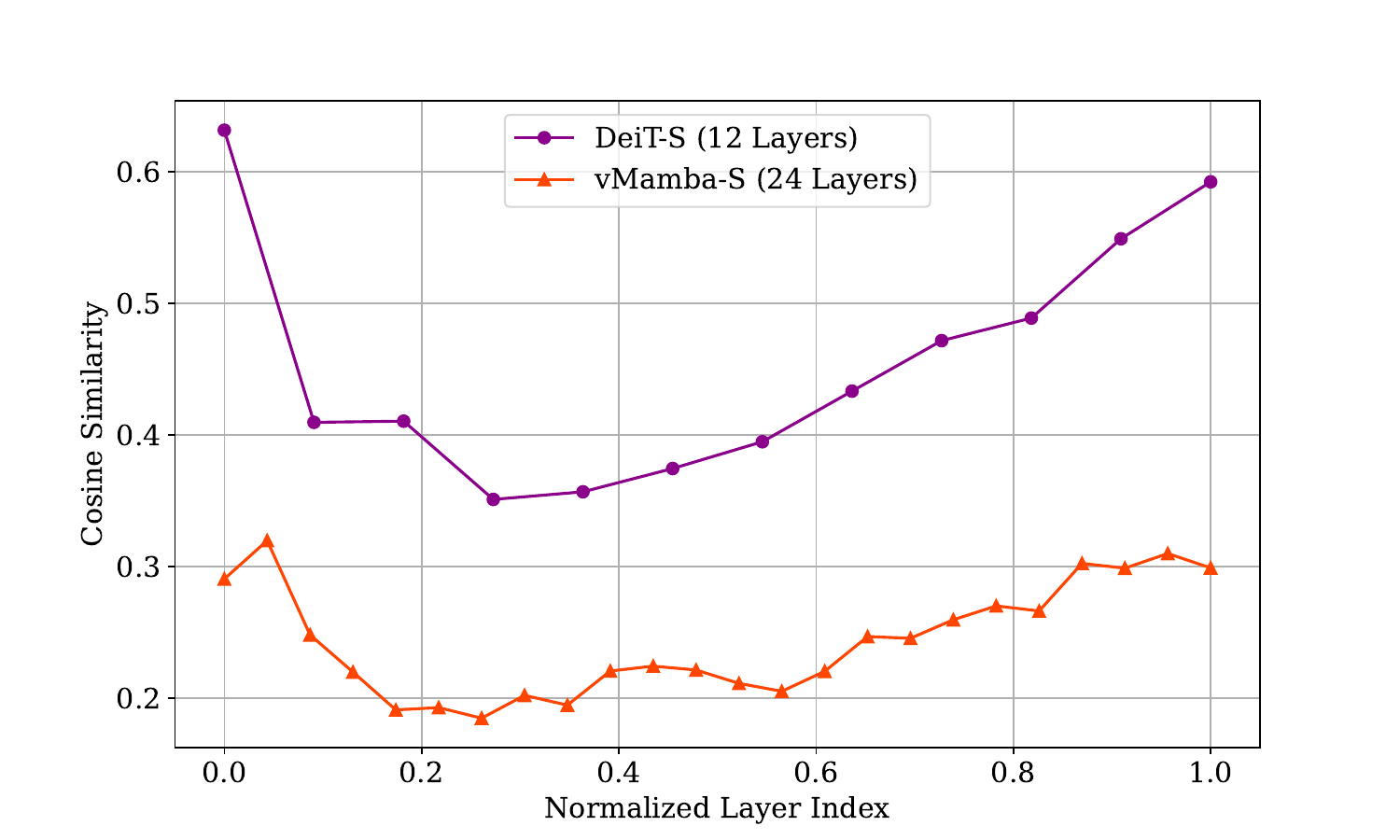}
\caption{The inter cosine-similarity across the tokens for both DeiT and Vision-Mamba small models across the layers, results are averaged across the whole validation set of ImageNet.}
\label{fig:cosine_sim}
\end{figure}

Recent transformer-based contributions~\cite{ali2023centered,wang2022anti} on over-smoothing have developed various tools for measuring it, which we utilize when comparing transformers to Mamba-based architectures. Following Wang et al~\cite{wang2022anti} we employ the cosine similarity metric, which operates on a layer index $l$ and its corresponding feature map $\boldsymbol{X}^l \in \mathbb{R}^{L+1\times d}$, such that $L$ is the sequence length excluding the CLS token and $d$ is the feature dimensionality. The cosine similarity for $\boldsymbol{X}^l$ is then defined as:
\begin{equation}
    \boldsymbol{M}^l_{\text{feat}} = \frac{2}{n(n-1)}\sum_{i = 1}^{n} \sum_{j=i+1}^{n} \frac{\boldsymbol{X}^{(l)T}_{i,:} \boldsymbol{X}^{(l)}_{j,:}}{\lVert\boldsymbol{X}^{(l)}_{i,:}\rVert_2 \lVert\boldsymbol{X}^{(l)}_{j,:}\rVert_2}\,,
\end{equation}
where $\boldsymbol{X}^{(l)}_{i,:}$ represents the $i$-th row (token) of $\boldsymbol{X}^{(l)}$. In words, this measure calculates the average similarity between all token pairs except for self-similarity. The results are averaged across the entire ImageNet validation set. For further details, refer to \cite{wang2022anti}.

Fig.~\ref{fig:cosine_sim} compares $\boldsymbol{M}^l_{\text{feat}}$ between transformers' attention and Mamba attention for different layers, where the layer index is normalized by the total number of layers. Evidently, there is a clear reduction in the inter-cosine similarity among tokens for all layers in the Mamba model (with 24 layers) compared to the Vision Transformer (12 layers). This variation can be attributed to the use of the softmax operator within the self-attention layer, a factor recognized as contributing to the oversmoothing problem~\cite{ali2023centered,wang2022anti}.

\section{NLP Experiments\label{app:nlpexpir}}
\begin{figure}

 \includegraphics[width=\linewidth]{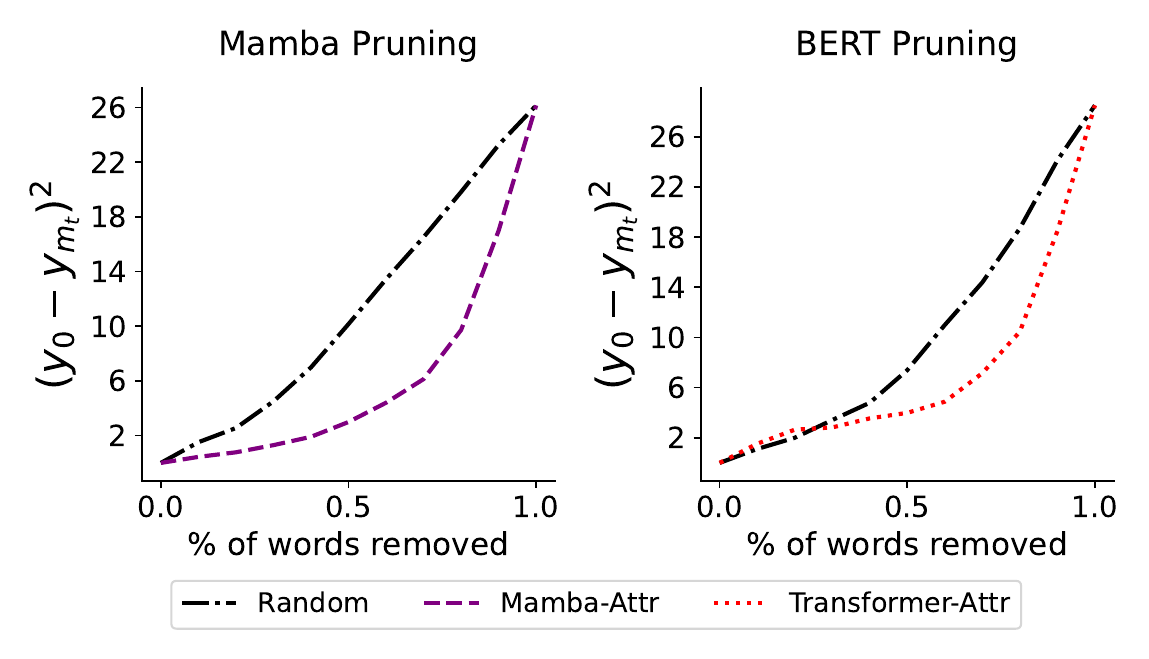} \\
 \includegraphics[width=\linewidth]{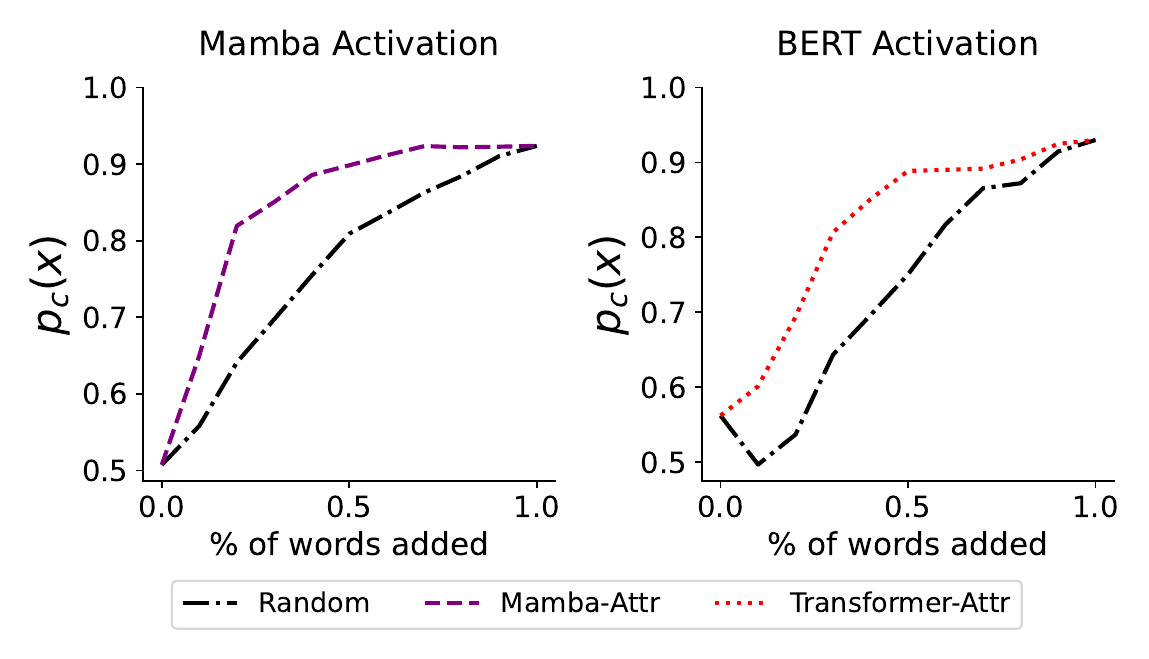} \\
\caption{Evaluation of explanations using input perturbations for the IMDb dataset, top row shows the results for the pruning task in which the words of least absolute relevance are replaced with <UNK> first and the bottom row shows the results for the activation task in which the most relevant words are added first, in both tasks we show the results for Mamba-Attr and Transformer-Attr separately.\label{fig:nlppert}}
\end{figure}
In this experiment, our aim is to extend the utilization of the proposed methods to the domain of Natural Language Processing (NLP). To achieve this, we conduct a comparative analysis between the Mamba-160M model and BERT-large, drawing upon established literature in the field~\cite{chefer2021transformer,ali2022xai}. Two settings are considered : (1) activation task, in this task, a good explanation involves listing tokens in order of their relevance, from most to least. When these tokens are added to an initially empty sentence, they should activate the network output as much and as quickly as possible. We evaluate the quality of explanations by observing the output probability $p_c(x)$ for the ground-truth class $c$. (2) pruning task, the pruning task involves removing tokens from the original sentence, starting with those deemed least relevant and progressing to the most relevant. We assess the impact of this pruning, by measuring the difference between the unpruned model's output logits $y_0$ and $y_{mt}$ of the pruned output.
In the activation task, we begin with a sentence containing "<UNK>" tokens and gradually replace them with the original tokens in order of highest to lowest relevance. Conversely, in the pruning task, we remove tokens from lowest to highest relevance by replacing them with "<UNK>" tokens.

The dataset employed in our study is the IMDb movie review sentiment classification dataset, consisting of 25,000 samples for training and an equal number for testing, with binary labels indicating sentiment polarity.
We utilize the Mamba-130M\footnote{\url{https://huggingface.co/trinhxuankhai/mamba_text_classification}} and BERT\footnote{\url{https://huggingface.co/textattack/bert-base-uncased-imdb}} models fine-tuned on the IMDB dataset for classification. BERT stands out as our baseline choice, benefiting from a readily available implementation of the Transformer-Attr method\footnote{\url{https://github.com/hila-chefer/Transformer-Explainability}}. Notably, both models exhibit comparable accuracy levels on the downstream task of IMDB movie review sentiment classification.
The results, depicted in Fig.~\ref{fig:nlppert}, illustrate that in both the pruning and activation tasks, Mamba-Attr exhibits comparable or occasionally superior performance to the Transformer-Attr method. We present the results of each method in separate graphs, as the two models are not directly comparable due to differences in the logit scale and the behavior on random changes to the prompt. 

In Sec~\ref{app:nlp_qual} we provide qualitative results for the different explanation methods (Mamba-Attr and Transformer-Attr) on the IMDb dataset, for both positive (green) and negative (red) sentiments. Evidently,  Mamba-Attr tends to generate more sparse explanations in comparison to its Transformer-Attr counterpart. For instance, in the analysis of the first negative sample, our method emphasizes the rating of "1" as the most salient feature along with other negative terms. Conversely, the transformer attribution method yields a less sparse explanation, focusing primarily on the relevant word while also encompassing other non-relevant terms. Similarly, in the assessment of the third negative example, our method exhibits a comparable behavior, placing emphasis on the ratings alongside other relevant negative terms. Conversely, while the salient words identified by the transformer attribution method remain valid, its explanation is comparatively less sparse. We observe a similar trend across positive sentiments as well (depicted in green). For instance, in the final positive review, Mamba-Attr distinctly highlights the phrase "Greatest Movie which ever made, " serving as clear evidence of a positive sentiment. In contrast, the explanation provided by Trans-Attr appears more broad and encompassing.

\section{NLP Qualitative Results}
\label{app:nlp_qual}
\begin{tabular}{ll} 
  \centering

  \textbf{Mamba-Attr} & \textbf{Transformer-Attr}

  \\ \\
  
  \includegraphics[width=0.5\textwidth]{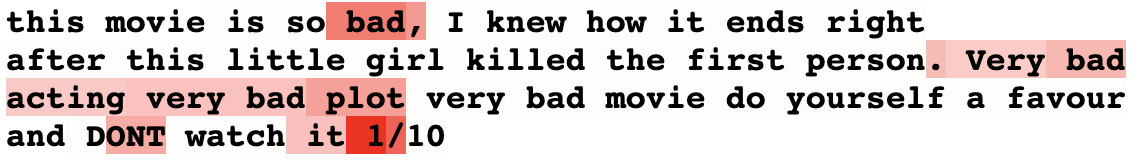} & \includegraphics[width=0.5\textwidth,height=27pt]{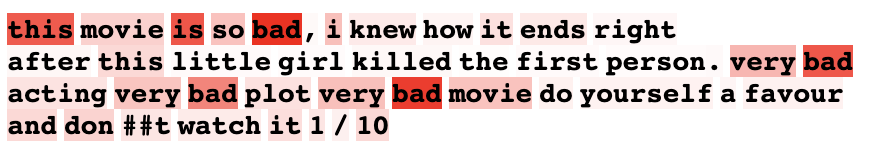}   \\
\\
  \includegraphics[width=0.5\textwidth]{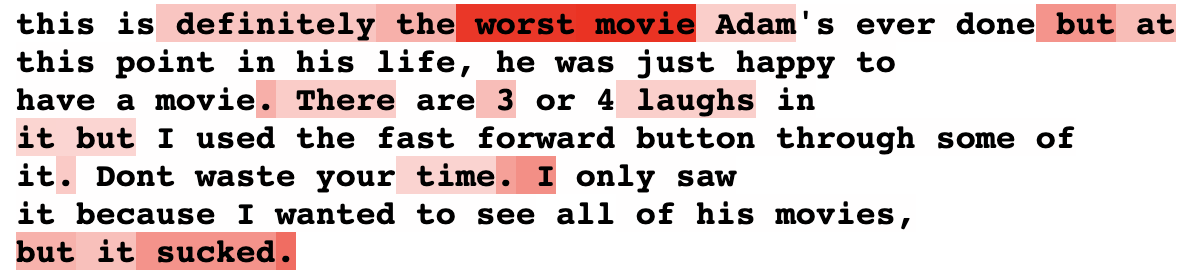} & \includegraphics[width=0.5\textwidth,height=42pt]{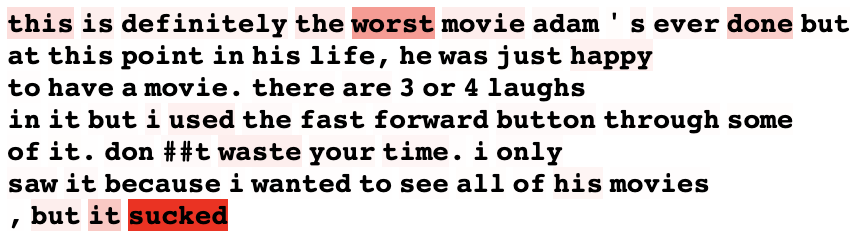}   \\

\\
  \includegraphics[width=0.5\textwidth,height=28pt]{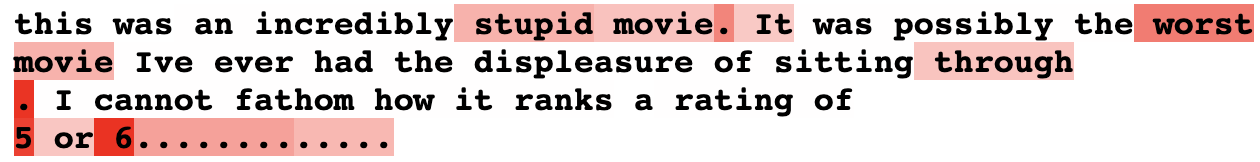} & \includegraphics[width=0.5\textwidth]{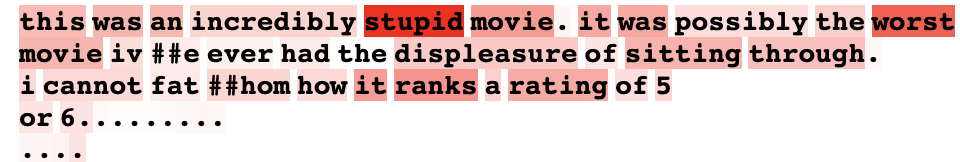}   \\

\\
  \includegraphics[width=0.5\textwidth]{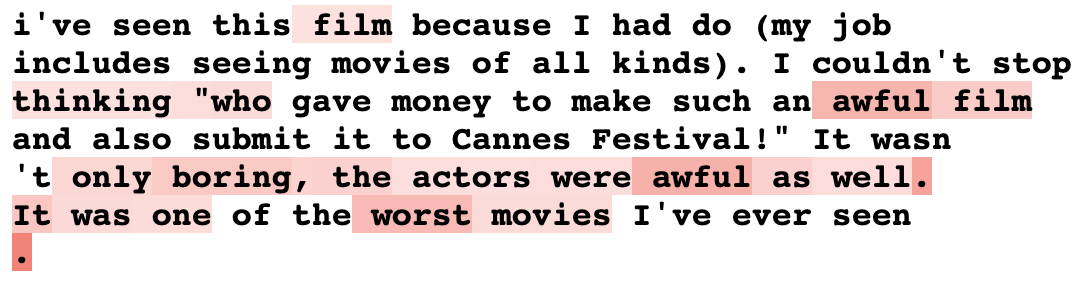} & \includegraphics[width=0.5\textwidth,height=48pt]{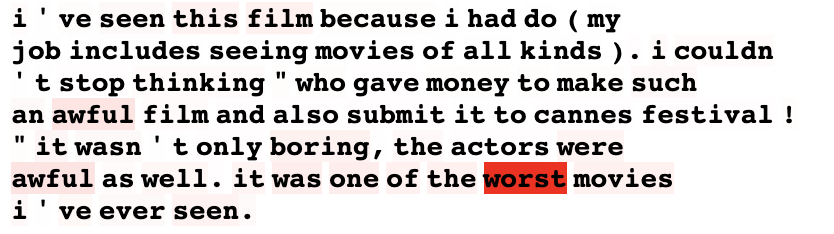}   \\

\\
  \includegraphics[width=0.5\textwidth,height=32pt]{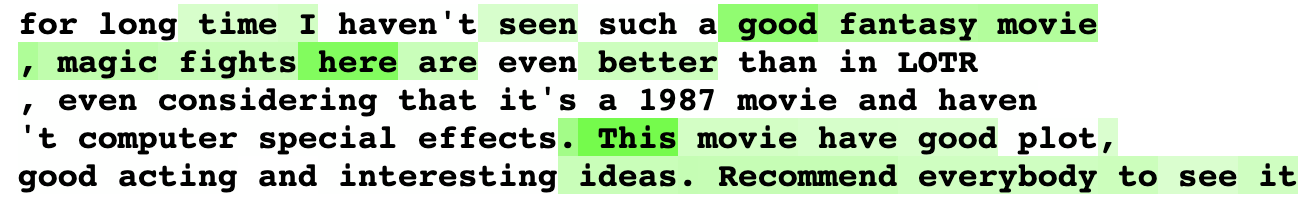} & \includegraphics[width=0.5\textwidth]{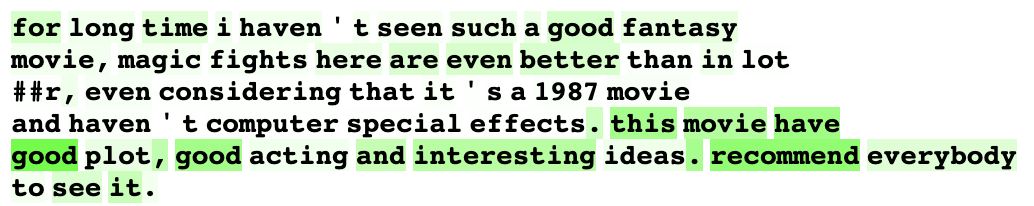}   \\

\\
  \includegraphics[width=0.5\textwidth]{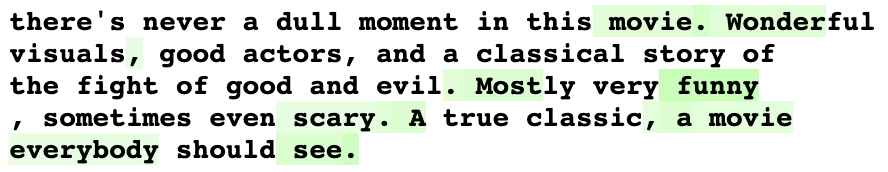} & \includegraphics[width=0.5\textwidth,height=36pt]{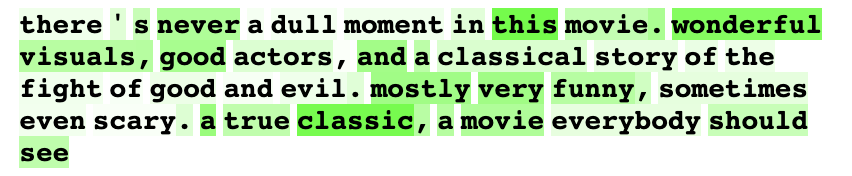}   \\

\\
  \includegraphics[width=0.5\textwidth]{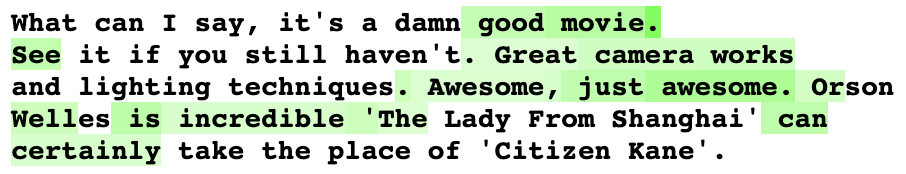} & \includegraphics[width=0.5\textwidth,height=36pt]{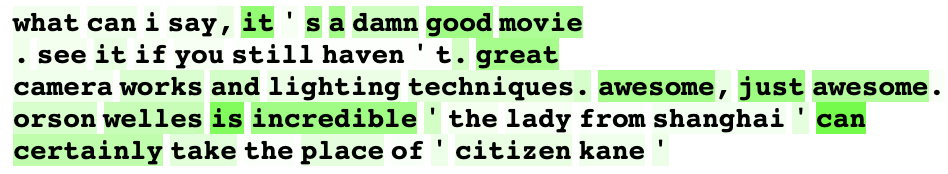}   \\

\\
  \includegraphics[width=0.5\textwidth,height=26pt]{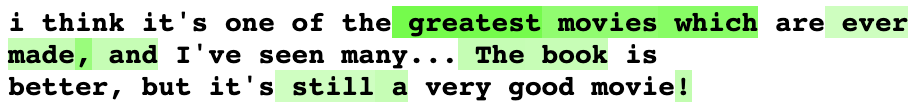} & \includegraphics[width=0.5\textwidth]{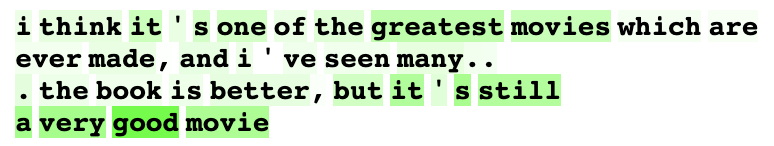}   \\
  
\end{tabular}

\section{Expressiveness of Mamba Models}\label{app:express}
\begin{theorem}
One channel of the selective state-space layer can express all functions that a single transformer head can express. Conversely, a single Transformer layer cannot express all functions that a single selective SSM layer can. 
\end{theorem}
\textbf{Motivation and Intuition:} The motivation for this proof relies on $\tilde{H}_{i,j}$ in Eq. 19, which enables Mamba to utilize continuous historical context within sequences more efficiently than traditional attention mechanisms. To exploit this capability, we focus on a problem involving input-dependent control over the entire input, a task that cannot be captured by relying solely on pairwise interactions at single layer, which constitute the foundation of self-attention.

{\noindent\textbf{Assumptions:}}
\begin{enumerate}
    \item For simplicity, we will disregard the discretization, as it has been shown to be unnecessary in previous work~\cite{gupta2022simplifying}.
    \item As our regime focuses on real elements $(x_i \in \mathbb{R})$, the hidden dimension of the transformer is 1. Hence, the parameters of both the self-attention mechanism and the Mamba are scalars, namely $A_i,B_i,C_i, W^Q,W^V,W^K \in \mathbb{R}$ 
\end{enumerate}

\begin{proof}
    Based on the definition of the count in row function, our proof straightforwardly arises from the following three lemmas:
    \begin{definition}
        The count in row problem: Given a binary sequence $x_1, x_2, \ldots, x_L$ such that $x_i \in \{0, 1\}$ for all $i \leq L$, the "count in row" function $f$ is defined to produce an output sequence $y_1, y_2, \ldots, y_L$, where each $y_i$ is determined based on the contiguous subsequence of 1s to which $x_i$ belongs. Formally:
        \begin{equation}    
        y_i = f(x_1, \ldots, x_i) = \max_{0\leq j \leq i} \Big{(} \{ i-j+1 \mid \prod_{k=j}^i [x_k > 0] = 1\} \cup \{ 0\} \Big{)}
        \end{equation}
 
    \end{definition}
    where $[x_k > 0]$ is the Iverson bracket, equaling 1 if $x_k> 0$ and 0 otherwise.
    \begin{lemma} One channel of Mamba can express the count in row function for sequences of any length.
    \end{lemma} 
    \begin{proof}
    Assumption 1 defines the following recurrence rule:
        \begin{equation} \label{eq:APPimeVariantMatrices1}
    \bar{B}_i = S_B (\hat{x}_i), \quad C_i = S_C (\hat{x}_i), \quad \bar{A}_i = S_{A}(\hat{x}_i) + A
\end{equation}

\begin{equation}\label{eq:TAPPiimeVariantMatrices2}
  h_t = \bar{A}_t h_{t-1} + \bar{B}_t x_t, \quad y_t = C_t h_t 
\end{equation}

By substituting $S_B, S_C, S_A = 1, A=0$ into Eq.~\ref{eq:TAPPiimeVariantMatrices2}, we obtain the following results:
\begin{equation}\label{eq:TAPPiimeVariantMatrices2sub1}
  h_t = h_{t-1} + x_t, \quad y_t = h_t 
\end{equation}
Now, there are two cases: (i) If $x_i = 0$, it's clear that both the state $h_t$ and the output $y_t$ receive zero values. (ii) Otherwise (if $x_i = 1$), we see that both $h_t$ and $y_t$ increase by one, clearly demonstrating that the entire mechanism exactly solves the count in row problem.

\end{proof}
\begin{lemma}
        One transformer head cannot express the count in row function for sequences with more than two elements.
\end{lemma}
\begin{proof}
        The self-attention mechanism computes the output as follows
\begin{equation}
O = \text{softmax}\left(\frac{(XW^Q)(XW^K)^T}{\sqrt{d_k}}\right) \cdot (XW^V)
\end{equation}
Consider the count in row problem for a binary sequence of length 3, the i-th coordinate in the output can be computed by:
\begin{equation}\label{eq:attentionPerCordProof}
O_i = \sum_{j=1}^{3} \left( \frac{\exp\left((W^Q \cdot x_i) \cdot (W^K \cdot x_j)\right)}{\sum_{k=1}^{3} \exp\left((W^Q \cdot x_i) \cdot (W^K \cdot x_k)\right)} \right) \cdot (W^V \cdot x_j)
\end{equation}

where we omitted the scale factor $\sqrt{d_k}$ (which can be incorporated into the $W^Q$ matrix).

For the sake of contradiction, we will assume that there are weights for the key, query, and value matrices that solve this problem. Furthermore, recall that $W^Q, W^K, W^V \in \mathbb{R}$, according to Assumption 2. Hence:

\begin{enumerate}
    \item For $(x_1,x_2,x_3) = (0,1,1)$, the output $y_3 = 2$. Plugging it into Eq.~\ref{eq:attentionPerCordProof} yields:
    \begin{equation}\label{eq:case1}
    O_3  = W^V \Big{(} \frac{2\exp(W^Q W^K)}{1+2\exp(W^Q W^K)} \Big{)}\ = 2
   \end{equation}
    \item For $(x_1,x_2,x_3) = (0,0,1)$, the output $y_3 = 1$. Plugging it into Eq.~\ref{eq:attentionPerCordProof} yields:
    \begin{equation}\label{eq:case2}
    O_3  = W^V \Big{(} \frac{\exp(W^Q W^K)}{2+\exp(W^Q W^K)} \Big{)}\ = 1
    \end{equation}

    Dividing Eq.\ref{eq:case1} by Eq.\ref{eq:case2} results in the following: 
    $$
    2 \frac{2+\exp(W^Q W^K)}{1+2\exp(W^Q W^K)}= 2 \quad \rightarrow \quad  \exp(W^Q W^K)=1
    $$
\end{enumerate}
Upon plugging it into the eq.~\ref{eq:case1}, we obtained:
$$ O_3 = W^V \frac{2}{3} = 2 \quad \rightarrow \quad W^V = 3$$

However, for $(x_1,x_2,x_3) = (1,0,1)$, the output $y_3$ is 1, by plugging it to eq.~\ref{eq:attentionPerCordProof}, and substituting the values of $W^V$ and $\exp(W^Q W^K)$, we obtain:
$$
O_3 = 3 \frac{2\exp(W^Q W^K)}{1+2\exp(W^Q W^K)} = 2 \neq 1
$$
As requested. Please note that the same technique also works when omitting the softmax function.
\end{proof}

\begin{lemma}
    One channel of the selective state-space layer can express all functions that a single transformer head can express.
\end{lemma}
\begin{proof}
For simplicity, we consider a causal attention variant without softmax, as the softmax is designed to normalize values rather than improve expressiveness. According to Assumption 1, we omit the discretization. Thus, we can simply set the value of $A_i$ to $\mathbb{I}$  which is the identity, by substitute $A = \mathbb{I}$ and $S_A = 0 $. Hence, it is clear that Eq. 11 and Eq. 12 become identical to causal attention, except for the softmax function. 
\end{proof}
\end{proof}

\section{Expressiveness of SSMs and Long-Convolution Layers}\label{app:mixing}
\begin{theorem}
(i) S4~\cite{gu2021efficiently}, DSS~\cite{dss}, S5~\cite{smith2022simplified} have fixed mixing elements. (ii) GSS~\cite{gss},and Hyena~\cite{poli2023hyena} have fixed mixing elements with diagonal data-control mechanism. (iii) Selective SSM have data-controlled non-diagonal mixers.
\end{theorem}
\begin{proof}
    We will prove this theorem separately per each layer:
    
    {\noindent\textbf{S4, DSS:\quad}} Both layers implicitly parametrize a convolution kernel $\bar{K}$ via the $A,\bar{B}$ and $\bar{C}$ matrices as follows: $$\bar{K} = (C\bar{B}, C\bar{A}\bar{B}, \cdots, C\bar{A}^{L-1}\bar{B} )$$
    This kernel does not depend on the input, and it is the only operation that captures interactions between tokens. Therefore, both layers have fixed elements.
    {\noindent\textbf{S5:\quad}} The S5 layer extend S4 such that it map multi-input to multi-output rather than mapping single-input to single-output. It use the following recurrent rule:
    \begin{equation}
        h_t = \bar{A}h_{t-1} + \bar{B}x_t,\quad y_t = C h_t,\quad \bar{A} \in \mathbb{R}^{P \times P}, \quad \bar{B} \in \mathbb{R}^{P \times H}, C \in \mathbb{R}^{H \times P}, x_t,y_t \in \mathbb{R}^H
    \end{equation}
    which can be computed by 
    \begin{equation}
        y_t = C \sum_{i=1}^t \bar{A}^{t-i}\bar{B}x_t
    \end{equation}
    However, in contrast to S4 and DSS, now $C\bar{A}^i \bar{B}$ in $\mathbb{R}^{H \times H}$ instead of in $\mathbb{R}$. Hence, we can conclude that the mechanism mixes tokens in a fixed pattern, which is captured by  $ C \sum_{i=1}^t \bar{A}^{t-i}\bar{B}x_t$.

    {\noindent\textbf{GSS:\quad}} GSS enhances the DSS framework, which utilizes fixed mixing elements, by incorporating an elementwise gating mechanism. Hence, the entire layer can be viewed as a composition of two operators, a mixer that isn't data-dependent (DSS), and an elementwise data-dependent gating, which is equivalent to a diagonal data-control linear operator.

    {\noindent\textbf{Hyena:\quad}} The Hyena layer is defined by the recurrence of two components: long implicit convolution and elementwise gating. For simplicity, we consider single recurrence steps to constitute the entire layer, since any layer can benefit from such a recurrent-based extension. Additionally, single recurrence is the most common application of the Hyena layer. Hence, similar to GSS, the layer can be viewed as a composition of a mixer that isn't data-dependent (based on CKConv~\cite{romero2021ckconv}) and a diagonal data-control operator, which is implemented through elementwise data-dependent gating. 
    
    {\noindent\textbf{Selective SSM:\quad}} As can be seen in Eq. 10 and 19, the selective SSM can be represented by:
\begin{equation}
y = \tilde{\alpha} x, \quad \tilde{\alpha}_{i,j} = \tilde{Q}_i \tilde{H}_{i,j}\tilde{K}_j
\end{equation}
Thus, it's clear that the linear operator, which relies on $\tilde{\alpha}$, is a data-controlled, non-diagonal mixer.

\end{proof}
\end{document}